\documentclass[9pt]{article}

\usepackage{fullpage}
\usepackage{amsmath}
\usepackage{amsthm,amsfonts,amssymb}
\usepackage{xspace}
\usepackage{bbm}
\usepackage{subcaption}
\usepackage{graphicx}
\usepackage{bbm}
\usepackage{hyperref}   
\usepackage{etoolbox}
\usepackage{enumitem}
\usepackage{authblk}
\usepackage{thmtools} 
\usepackage{thm-restate}

\usepackage[noend]{algorithmic}
\usepackage[ruled,vlined]{algorithm2e}
\usepackage{comment}
\usepackage{natbib}

\usepackage{float}

	\hypersetup{
			colorlinks=true,
			linkcolor=[rgb]{0,0,.5},
			urlcolor=[rgb]{0,0,.5},
			citecolor=[rgb]{0,0,.5},
			pdfstartview=FitH}

\usepackage{accents}

\usepackage{graphicx} 
\usepackage{booktabs} 

\usepackage{caption}

\usepackage{xcolor}

\usepackage{pgfplots}

\pgfplotsset{compat=1.10}

\usepackage{tikz}

\usepackage{enumitem}
\usepackage{mathtools}
\usetikzlibrary{quotes,positioning}
\allowdisplaybreaks

\newtheorem{theorem}{Theorem}[section]
\newtheorem{lemma}{Lemma}[section]
\newtheorem{corollary}{Corollary}[section]
\newtheorem{definition}{Definition}[section]
\newtheorem{remark}{Remark}[section]

\newcounter{numrellocal}
\renewcommand{\thenumrellocal}{\arabic{numrellocal}}
\newcounter{numrelglobal}

\makeatletter
\newcommand{\numrel}[2]{
  \stepcounter{numrellocal}
  \refstepcounter{numrelglobal}
  \ltx@label{#2}
  \overset{(\thenumrellocal)}{#1}
}
\makeatother

\usepackage{xcolor}

\newcommand{\E}{\mathop{\mathbb{E}}}
\newcommand{\argmin}{\mathop{\mathrm{argmin}}}
\newcommand{\argmax}{\mathop{\mathrm{argmax}}}
\newcommand{\cZ}{\mathcal{Z}}
\newcommand{\cX}{\mathcal{X}}
\newcommand{\cY}{\mathcal{Y}}

\newcommand{\cG}{\mathcal{G}}

\newcommand{\cD}{\mathcal{D}}

\usepackage{thmtools} 
\usepackage{thm-restate}

\usepackage{amsmath}
\usepackage{pgfplots}

\title{Reconciling Individual Probability Forecasts}

\author[1]{Aaron Roth}
\author[2]{Alexander Tolbert}
\author[2]{Scott Weinstein}
\affil[1]{Department of Computer and Information Sciences, University of Pennsylvania}
\affil[2]{Department of Philosophy, University of Pennsylvania}

\begin{document}
\maketitle
\begin{abstract}
Individual probabilities refer to the probabilities of outcomes that are realized only once: the probability that it will rain tomorrow, the probability that Alice will die within the next 12 months, the probability that Bob will be arrested for a violent crime in the next 18 months, etc. Individual probabilities are fundamentally unknowable. Nevertheless, we show that two parties who agree on the data---or on how to sample from a data distribution---cannot agree to disagree on how to model individual probabilities. This is because any two models of individual probabilities that substantially disagree can together be used to empirically falsify \emph{and improve} at least one of the two models. This can be efficiently iterated in a process of ``reconciliation'' that results in models that both parties agree are superior to the models they started with, and which themselves (almost) agree on the forecasts of individual probabilities (almost) everywhere. We conclude that although individual probabilities are unknowable, they are \emph{contestable} via a computationally and data efficient  process that must lead to agreement. Thus we cannot find ourselves in a situation in which we have two equally accurate and unimprovable models that disagree substantially in their predictions---providing an answer to what is sometimes called the predictive or model multiplicity problem.
\end{abstract}


\section{Introduction}
Probabilistic modelling in machine learning and statistics predicts ``individual probabilities'' as a matter of course. In weather forecasting, we speak of the probability of rain tomorrow; in life insurance underwriting we speak of the probability that Alice will die in the next 12 months; in recidivism prediction we speak of the probability that an inmate Bob will commit a violent crime within 18 months of being released on parole; in predictive medicine we speak of the probability that Carol will develop breast cancer before the age of 50 --- and so on. But these are not repeated events: we have no way of directly measuring an ``individual probability'' --- and indeed, even the semantics of an individual probability are unclear and have been the subject of deep interrogation within the philosophy of science and statistics \citep{hajek2007reference,dawid2017individual} and theoretical computer science \citep{dwork2021outcome}. Within the philosophy of science, puzzles related to individual probability have been closely identified with ``the reference class problem'' \citep{hajek2007reference}. This is a close cousin of a concern that has recently arisen in the context of fairness in machine learning called the ``predictive multiplicity problem'' (a focal subset of ``model multiplicity problems'')  \citep{marx2020predictive,black2022multiplicity} which \cite{breiman2001statistical} earlier called the ``Rashomon Effect''. At the core of both of these problems is the fact that from a data sample that is much smaller than the data universe (i.e. the set of all possible observations), we will have observed at most one individual with a particular set of characteristics, and at most one outcome for the event that an ``individual probability'' speaks to: It will either rain tomorrow or it will not; Alice will either die within the next year or she will not; etc. We do not have the luxury of observing a large number of repetitions and taking averages. 

\cite{dawid2017individual} lays out two broad classes of perspectives on individual probabilities: the \emph{group to individual} perspective and the \emph{individual to group} perspective. The group to individual perspective is roughly as follows: We cannot measure individual probabilities from data, but we \emph{can} measure averages of outcomes within sufficiently large \emph{reference classes} $S$. A reference class $S$ is just some well defined subset of the observed data: for example (in the weather forecasting setting) the set of days in which there is cloud cover and humidity is above 60\%, or (in the life insurance setting) the set of 65 year old women with a history of high blood pressure. Given a reference class that is large enough that we have observed in our data many members of the reference class, we can empirically estimate the prevalence of the outcome we are concerned with forecasting (rain, death within 12 months) for members of the reference class. Then, if we are asked to forecast an individual probability (the probability that \emph{Alice} will die within the next 12 months), we simply pick an appropriate reference class $S$ such that $\text{Alice} \in S$ and then respond with the proportion of observed deaths within a 12 month period for individuals from reference class $S$. The principal problem with this approach (known as the ``reference class problem'' \citep{hajek2007reference}) is that Alice will simultaneously be a member of many different reference classes $S$. We cannot condition on \emph{everything} we know about Alice, or we will end up with a reference class that does not contain enough examples for us do statistical inference on: thus we must pick and choose. But should we have conditioned on her age, gender and blood pressure? What about her weight? Her job? Her marital status? Her vaccination history? Defining reference classes with respect to different subsets of these attributes will generally lead to different estimates for the probability that Alice will die within the next 12 months: what privileges one of these estimates over another? This is the reference class problem. 

On the other hand, the \emph{individual to group} perspective treats individual probabilities as the first class objects. This is the perspective most familiar in machine learning and statistics: models $f$ are learned from data with the goal of mapping individuals (e.g. ``Alice'') to individual probabilities for the outcome of interest, $f(\text{Alice})$.\footnote{Here of course what is input to the model is some \emph{representation} of the individual, encoding the information that we have available about them.} Models of individual probabilities can also be aggregated over to give predicted probabilities conditional on reference classes. If we want to evaluate the probability of an outcome conditional on some reference class $S$, we can do so by averaging the model's predictions over individuals in $S$. We cannot measure individual probabilities, but from data we can measure the average probability of an outcome over a sufficiently large reference class $S$, which gives us a way to empirically falsify a model $f$ from data: if the prediction implied by $f$ for the average outcome conditional on a large reference class does not match the average outcome we can measure from the data, then the model $f$ must be wrong. Multicalibration, introduced by \cite{hebert2018multicalibration}, gives us a way to build models of individual probabilities that are consistent with the data for large numbers of arbitrarily chosen reference classes $S$ --- i.e. models that are not empirically falsified by any of the pre-specified reference classes. Nevertheless, multi-calibrated models are not unique\footnote{If a model is \emph{perfectly} multicalibrated with respect to every reference class (including singletons that contain only single individuals), then it must encode the true individual probabilities, but this cannot be approximated from data unless many samples of every possible representation of an individual have been observed, which is infeasible in realistic high dimensional inference settings.  Similarly, \cite{dawid85} proved that probabilities generated by computable forecasters that are multicalibrated with respect to every computable reference class must be unique, but only in a limiting sense --- two such multicalibrated models may substantially disagree on any finite set of data. The kinds of multicalibration guarantees that are obtainable from samples of data that are polynomial in the data dimension \citep{hebert2018multicalibration,gupta2022online} do not lead to unique models.}: we can have multiple models that have large disagreements in many of their individual predictions that nevertheless are equally consistent with the data on a large collection of reference classes. This is an instance of the \emph{predictive multiplicity} problem \citep{marx2020predictive,black2022multiplicity,breiman2001statistical}.

The predictive multiplicity problem is usually not phrased in terms of multicalibration and reference classes, but in terms of \emph{accuracy} or \emph{error}. If a model encodes true individual probabilities, then it will minimize expected squared error\footnote{or any other proper scoring rule.} amongst all possible models. Moreover,  expected squared error is something that we can efficiently estimate from data. Hence, if we have two models $f_1$ and $f_2$, and we can infer from data that $f_1$ has lower expected squared error than $f_2$, then this is an empirical falsification of the hypothesis that $f_2$ correctly encodes individual probabilities. This serves as a normative justification for selecting amongst models based on their accuracy, which is a common practice. The predictive  multiplicity problem arises when we have two models $f_1$ and $f_2$ (and perhaps others) that are equally accurate, but disagree substantially on many of their predictions.  More generally the predictive multiplicity problem arises when we have multiple models that differ substantially in their predictions, but are seemingly equally consistent with the data before us. 

Despite arising from different conceptions of individual probability, the reference class problem and the predictive  multiplicity problem result in the same practical concern: that data do not encode unique estimates for the individual probabilities for many  individuals. If this is the case, then what justification do we have in making consequential decisions as a result of predictions that our models make about individual probabilities? How can we justify setting a high rate for Alice's life insurance, denying parole to Bob, or suggesting life-altering preventative surgery to Carol based on the predictions of some model $f_1$ if we have an equally good (and equally well supported by the data) model $f_2$ that makes predictions that would lead us to take the opposite course of action? 

\subsection{Our Results}
We show that given a common understanding of the data (or the process of sampling from the data distribution), models of individual probabilities are \emph{contestable} through an efficient model reconciliation process that must lead to broad agreement. Specifically, suppose one party $A$ proposes a model of individual probabilities $f_A$, that another party $B$ thinks is flawed. $B$ can \emph{contest} $f_A$ by proposing their own model of individual probabilities $f_B$. There are two possible outcomes:
\begin{enumerate}
    \item $f_A$ and $f_B$ agree in their predictions almost everywhere.\footnote{We use the expressions “almost everywhere” and “almost agree” as shorthand for quantitative statements that are made explicit in the formal presentation of our results. Insofar as we are working in the context of discrete distributions, it should be clear that we are not using these expressions in their usual measure-theoretic sense. We note that our focus on discrete distributions is merely to avoid dealing with measure-theoretic niceties, and is not essential to any of our results} In this case, it turns out there was no substantial disagreement.
    \item $f_A$ and $f_B$ substantially disagree in their predictions for a large portion of the population. 
\end{enumerate}
In the second case, we can efficiently extract from the disagreement region of $f_A$ and $f_B$ a large reference class $S = S(f_A,f_B)$ such that on this reference class, not only do $f_A$ and $f_B$ disagree on individual predictions, they also disagree substantially on their prediction of the average outcome conditional on membership in $S$. Because $S$ is large, from only a modest amount of data, we can accurately estimate the average outcome conditional on $S$. But because $f_A$ and $f_B$ have a substantial disagreement about this quantity, our measurement is guaranteed to falsify at least one of the two models. 

Suppose it is model $f_A$ that is falsified. Then, using a very simple and efficient model update operation of the same sort used for computing multicalibrated models \citep{hebert2018multicalibration}, we can update $f_A$ to produce a new model $f_A'$ that now makes predictions that are correct on average over $S$. The new model $f_A'$ is guaranteed to have significantly reduced squared error compared to $f_A$, and so is a better model not only in that it has not yet been falsified, but in that it is more accurate. 

After this update, we can then repeat the process: Either $f_A'$ and $f_B$ agree on their predictions almost everywhere, or we can again falsify one of the models and improve it using a large reference class $S' = S(f_A', f_B)$. The only way for this process to end is with two models that agree in their predictions almost everywhere. Moreover, because each iteration of falsification and improvement improves the expected squared error of at least one of the two models, the process cannot continue for very many iterations --- fast agreement of the models is guaranteed. 

In Section \ref{sec:insample}, we formally derive the guarantees of our model reconciliation process under the assumption that we can directly evaluate conditional outcome probabilities conditional on large reference classes $S$: this makes our analysis more transparent. Then, in Section \ref{sec:outofsample}, we show that we can run our model reconciliation process on the empirical distribution over a modestly sized dataset that is \emph{sampled} i.i.d. from some unknown underlying distribution, and that its guarantees carry over to the unknown distribution of interest. Here ``modestly sized'' means a number of samples that is \emph{independent} of the complexity of the models to be reconciled or the dimension (or any other property) of the underlying distribution, and that depends only polynomially on the quantitative parameters controlling how closely we want the models resulting from the reconciliation process to agree. In fact, as we observe, the data required to guarantee agreement on a $1-\alpha$ fraction of the data distribution is not much larger than the data requirement that would be necessary for two parties to agree (from samples) on the average outcome probability conditional on a \emph{single} reference class representing an $\alpha$ fraction of the data.

\subsection{Discussion}

\paragraph{Are ``Individual Probabilities'' Coherent?} What are we assuming when we model the world using ``individual probabilities''? Are we assuming some kind of idealized, unrealized randomness, which is simply a poor stand-in for our ignorance of the relevant processes? No. Our modelling choices do not preclude a deterministic world: all true individual ``probabilities'' could be $0$ or $1$, simply recording the outcomes of interest. 
 We still allow for \emph{models}
to predict non-integer probabilities; we can view these as simply expressing uncertainty about
the outcomes, or as encoding objective features of a stochastic universe. Our results imply that after reconciliation, two parties must agree about their
assignments of individual probabilities, regardless of the philosophical commitments each may have about the nature of probability.

\paragraph{Agreement is Guaranteed; Not Truth} We emphasize that individual probabilities are not uniquely determined from observed data. Consider a toy model of weather forecasting in which features $x$ encode the date, and outcomes $y$ encode whether or not it rains on that date. The following two situations are observationally indistinguishable: 
\begin{enumerate}
    \item On every day $x$, the individual probability of rain $p(x) = 1/2$, and
    \item Before the start of time, God selected a subset of days uniformly at random to have individual probability of rain $p(x) = 1$ and the remaining set to  have individual probability of rain $p(x) = 0$. 
\end{enumerate} 

There is no hope of distinguishing these two situations from data about outcomes alone, and so it is plainly
impossible to learn a model that is guaranteed to accurately encode individual probabilities from such data\footnote{Of course, we do not take such radical under-determination of individual probability assignments  from data about outcomes alone to in anyway impugn the objectivity of such assignments. Indeed, the primary virtue of our results from a philosophical perspective is that they provide an efficient method to guarantee inter-subjective agreement about individual probability assignments and thus secure their objectivity to this extent.}— in
fact, it is not clear that this goal is meaningful, as they are not uniquely determined.\footnote{It is perhaps worth remarking that in this case Dawid's uniqueness result \citep{dawid85}, cited earlier, implies, with probability one with respect to God's choices, that there \textit{is} an asymptotically unique \textit{computable} assignment of probabilities that is computably calibrated with the data generated by God's choices. There is no paradox here: insofar as God's choices are generated uniformly at random, her (deterministic) probability forecast is, with probability one, algorithmically random, and thus not computable.} Nevertheless, suppose we believed that the individual probability of rain was $p(x) = 1/2$ every day. If we met a forecaster who was able to make more accurate predictions (i.e. predictions that had lower squared error) on previously unobserved data, we would be forced to recognize that our model was incorrect --- because we could \emph{compare} the performance of the two models on data. This drives our result (and similar work on multiple expert testing \citep{al2008comparative,feinberg2008testing} --- see Section \ref{sec:related}), and is the reason that we can guarantee \emph{agreement} rather than \emph{truth}. Nevertheless, the updates that result from our reconciliation process always move towards truth---because they are error improving---but they stop when the available models agree, which might be well before truth is attained.  

\paragraph{Predictive Multiplicity Comes from Restricting Model Classes}
Previous work has empirically noted and quantified the phenomenon of predictive  multiplicity---i.e. that solving an error minimization problem over some class of models can result in multiple solutions of (roughly) equivalent error \citep{d2020underspecification,marx2020predictive}. How do these results square with our contention that the predictive  multiplicity problem cannot arise, because two equally accurate \emph{but substantially different} models constructively imply the existence of a more accurate model? 

The answer is that predictive  multiplicity can arise when models are restricted to lie within some pre-specified hypothesis class, like linear threshold functions, bounded depth decision trees, or neural networks with a particular architecture. Traditionally machine learning is done by optimizing a model within a fixed model class, and this is the setting in which predictive multiplicity has been empirically observed and quantified. In contrast, our algorithm for reconciling pairs of models $f_1$ and $f_2$ produces a model $f_3$ that need not lie in the same model class as $f_1$ and $f_2$. This is key to side-stepping the predictive multiplicity problem. Traditional methods in machine learning and statistics optimize over models from restricted classes to avoid the problem of overfitting. In contrast, we avoid overfitting despite not restricting our model classes \textit{a priori} by bounding the \emph{number} of updates that can occur through our reconciliation process.

\paragraph{What about Kasey?} Suppose we have a model that has gone through our reconciliation process with another model for the same task, and so we can promise that (say) it predicts the same individual probability (plus or minus $1\%$) as the other model for 99\% of individuals. We then use this model to predict an individual probability about Kasey, which we use to inform a consequential decision. Suppose Kasey is among the 1\% of people for whom the models disagree. Kasey's main interest is that there should be no ambiguity about his individual prediction --- why should he be satisfied that we can promise that the models must agree \emph{almost} everywhere, if they do not agree on their predictions for him?\footnote{Adapted from an insightful question asked of us by Konstantin Genin.}

We can say nothing about Kasey's true individual probability.  Nevertheless, what we can promise is that there are not many people in Kasey's position---such that there are two models equally consistent with the data  that disagree on their individual predictions for that individual. The result is that for people like Kasey, we can afford to devote additional resources towards decision making that we could not afford if Kasey's situation was more common. For example, for the 1\% of the population for whom the models are not in agreement, we can imagine devoting more time, human expertise, and deliberation than we could for the general population---perhaps by escalating these cases to a panel of experts. Alternately, in a setting in which Kasey has a clearly preferred outcome, we might be able to afford to give him the benefit of the doubt --- for example, if our best models disagree on Kasey's individual probability in a decision relevant way, we could commit to using the most favorable of the predictions for Kasey. Again, we can afford to act in this way because there are only few such individuals\footnote{More specifically, when our algorithm is run on a sample of people drawn from a distribution, we promise that we produce models that both agree on their predictions on 99\% of people in the sample, and 99\% of the probability mass of the distribution from which the sample was drawn. If we imagine the sampling distribution as being defined over some much larger finite ``population'', then this corresponds to agreement on 99\% of the population \emph{if the sampling distribution is uniform over the underlying population}. If the sampling distribution is not uniform (either by design, for example, in the case that the feature domain $\cX$ used to model the underlying population (see Section 2) is countably infinite, or because of an error in the sampling procedure), then the guarantee of agreement on 99\% of the mass of the sampling distribution need not correspond to agreement on 99\% of the underlying population any longer, and there may be ``many such'' people on whom the model disagrees even though they have only small probability mass in the sampling distribution. In general, in our exposition, we will make the assumption that the sampling distribution correctly represents the population we are interested in.}. 
Moreover, our model reconciliation process can drive the fraction of the population on which the models disagree to $0$ at the cost of access to more data --- although of course with finite amounts of data we will never entirely eliminate disagreements. 

\paragraph{Do models really predict individual probabilities?} Another objection we can imagine is that in the settings we discuss---weather prediction, life insurance underwriting, recidivism prediction, etc.---it is logically impossible to observe repeated trials, because tomorrow will only occur once, Alice has only one life to live, and so on. In contrast, when we move to the formalism of a probability distribution over representations of individuals, it may be extremely unlikely (or even a measure 0 event) to observe the same representation of an individual multiple times, but it is no longer a logical impossibility. Said another way, when we model individuals in some representation space, we may fail to record idiosyncratic details of the individual, and so we are no longer speaking of individual probabilities, but rather average outcomes over the reference class defined by people who share the same representation. But this is not a sharp distinction, because our results have no dependence at all on the dimensionality or complexity of the representation we use for individuals. For this objection to have teeth, it must be that there is some crucial idiosyncrasy of an individual that we have failed to capture in our representation: if so, add this to our representation! Our results remain the same (not just qualitatively but also quantitatively) even if the representation of every individual records the position of every molecule in their body, a complete history of their life from birth until the present, or anything else, and so does not rely even implicitly on having only an impoverished representation of an individual to work with rather than ``the real thing''.

\subsection{Additional Related Work}
\label{sec:related}

Our work is related to a number of strands of literature across statistics, economics, and computer science. \cite{nau1991arbitrage} showed that agents interacting in a market that has no arbitrage opportunities must agree on the probabilities of all payoff relevant outcomes, because if there are two agents who disagree on such probabilities, then a third party has the ability to engage in transactions with the two of them that are guaranteed to be profitable no matter what the outcome, with the magnitude of the profit scaling with the number and magnitude of the disagrements about outcome probabilities.  \cite{aumann1976agreeing} proved that two Bayesians who share a common prior, but may have made different observations, must agree on the posterior expectation of a random variable if their posterior distributions are common knowledge.  Although Aumann's original result was nonconstructive, subsequent work has shown that agreement can be reached with finite, communication efficient protocols \citep{geanakoplos1982we,aaronson2005complexity}. Despite similarity in its conclusions, this line of work is quite distinct from ours. In the Bayesian setting that this line of work focuses on, it is immediate that two  agents who share the same set of observations and prior beliefs must share the same posterior beliefs (as a posterior distribution is determined, via Bayes rule, as a function only of the prior distribution and observations). Aumann's agreement theorem instead shows that if agents have arrived at common knowledge of their posterior distributions, then their posteriors must agree \emph{even if they have not directly shared their observations}. In contrast, in a frequentist setting, individual probabilities are not uniquely determined from data, which forms the basis of the reference class problem\footnote{In a Bayesian framework, problems that are similar to the reference class problem emerge in making ones choice of priors \citep{hajek2007reference}.} \citep{hajek2007reference} and the model multiplicity problem \citep{black2022multiplicity}. Our work considers how two frequentist agents who agree on the same set of data (or the distribution from which it was drawn) must come to agree on individual probabilities --- a problem which would not arise in the first place if they were Bayesian agents with a common prior.

\cite{dawid1982well} proposed calibration as a desirable frequentist condition for evaluating probabilistic forecasts: roughly speaking that the outcome being forecast should have appeared with empirical frequency $p$ conditional on the forecaster predicting probability $p$ of the outcome, simultaneously for all predictions $p$. Subsequently, \cite{dawid85} studied a substantial strengthening of this condition called \emph{computable calibration} that requires calibration to hold simultaniously on all computable subsets of the data.\footnote{Dawid explicitly links calibration as a criterion of adequacy for a forecasting model to the notion of randomness - a forecaster $f$ is calibrated with respect to a data sequence $s$ if and only if $s$ is random with respect to the probability distribution induced by $f$ on the collection of all data sequences -- if a probabilistic model is a correct explanation of the data, the data must not look atypical, that is non-random, with respect to the model. Of course, in order to endow this notion with explicit mathematical content, one requires a suitable mathematical theory of randomness. The notion of computable calibration adopted by Dawid draws on a notion of algorithmic randomness 
elaborated by von Mises, Wald, and Church in the 1930's. This notion was shown to be defective by \cite{Ville} who established that there are infinite sequences of 0's and 1's, every initial segment of which contain more 0's than 1's, that are nonetheless random in the sense of von Mises, Wald, and Church. A far more credible algorithmic notion of randomness which avoids this difficulty was introduced by \cite{Martin-Lof66} based on the idea that a sequence is random if and only if it passes every one of a comprehensive class of computably enumerable statistical tests. (An excellent account of these developments may be found in Chapter 6 of \cite{DowneyHirschfeldt}.)  Dawid notes the deficiency of the von Mises, Wald, Church notion of randomness, and suggests the desirability of investigating a notion of calibration based on Martin-L\"{o}f randomness. \cite{VOVK20102632}  investigates such a notion. \cite{bienvenuetal} studies further generalizations in this direction which apply to the case of real-valued functions.} Dawid proved that in the infinite data limit, two computably calibrated forecasters must approximately agree in their predictions almost everywhere --- that is, except on a finite subset of the data \citep{dawid85}. He notes explicitly that this criterion is not of practical use in finite data scenarios, and speculates about the desirability of restrictions of computable calibration to finite sample scenarios (anticipating \emph{multicalibration} \citep{hebert2018multicalibration}). Multicalibration \citep{hebert2018multicalibration} asks for calibration on a restricted class of subsets of the data. \cite{hebert2018multicalibration} gave algorithms for learning multicalibrated predictors with data requirements that scale only modestly with the number of subsets of the data on which calibration is required (and efficient algorithms whenever it is possible to efficiently optimize over these subsets)---but multicalibrated forecasts need not be unique. \cite{jung2021moment} generalized multicalibration (which aims to be consistent with \emph{mean} outcomes) to moments and other properties of real valued outcomes, and gave efficient algorithms for obtaining these guarantees.  \cite{dwork2021outcome} generalized multicalibration to notions of ``outcome indistinguishability'' that ask that a probabilistic forecaster be indistinguishable from a true probabilistic model with respect to a hierarchy of distinguishers that might have access not just to the predictions but to the implementation details of the forecaster itself. \cite{dwork2021outcome} explicitly connect outcome indistinguishability to philosophical questions surrounding individual probabilities. Multicalibration has proven to be an effective technique for improving individual predictions in several applications in predictive medicine \citep{barda2020developing,barda2021addressing}

\cite{foster1998asymptotic} gave the first algorithm to constructively make predictions of individual probabilities guaranteed to generate calibrated forecasts against arbitrary sequences of outcomes (and so necessarily without any knowledge of the ``true individual probabilities'', since the outcomes can be generated adversarially, with knowledge of the predictor's algorithm).  \cite{sandroni2003calibration} show constructively how to achieve calibration in the infinite data limit on any computable subsequence of an arbitrary sequence of outcomes. \cite{sandroni2003reproducible} showed that \emph{any} empirical test (not just calibration tests) that is guaranteed to pass an expert who is forecasting true individual probabilities can be passed by a prediction algorithm on any sequence of outcomes. This is closely related to the fact that individual probabilities are not uniquely specified by data --- and so we cannot attempt to test an expert by computing unique individual probabilities ourselves.  \cite{gupta2022online} gave computationally and sample efficient algorithms for achieving multi-calibrated forecasts against arbitrary sequences of outcomes --- for means, moments, and quantiles. \cite{bastani2022practical} gave practical implementations of quantile multicalibration algorithms in adversarial sequential settings, and applied them to give algorithms for producing prediction sets of various kinds of classifiers with calibrated, group-wise conditionally valid guarantees.

Although \cite{sandroni2003reproducible} showed that no empirical test of outcomes can distinguish a forecaster with knowledge of true individual probabilities from one without such knowledge in \emph{isolation}, \cite{al2008comparative} and \cite{feinberg2008testing} showed that there are \emph{comparative tests} that can distinguish between \emph{two} forecasters, one of whom is forecasting true individual probabilities and one of whom is not. In particular, the test of \cite{feinberg2008testing} is based on checking for \emph{cross-calibration} between two forecasters --- i.e. calibration conditional on the predictions of \emph{both} forecasters, and is driven by the fact that on a sequence of predictions such that one forecaster predicts a probability for an outcome $p$ and the other predicts a probability $p' \neq p$, they cannot both be right, which is empirically verifiable if there are many such rounds. In the context of studying the utility of predictors for downstream fairness interventions, \cite{garg2019tracking} study predictors that are \emph{refinements} of one another (in the sense of \cite{degroot1981assessing,degroot1983comparison}). They give a simple algorithm (``Merge'') that given any two predictors $f_1,f_2$, produces a predictor $f_3$ that is cross-calibrated with respect to $f_1$ and $f_2$, and hence is a refinement of both. A variant of the ``Merge'' algorithm of \cite{garg2019tracking} could be used in place of our ``Reconcile'' algorithm in our arguments; the two algorithms have incomparable data requirements, but would lead to the same qualitative conclusions. 

\cite{biasbounties} proposes a framework in which models that are sub-optimal on different subsets of the population can be updated and improved as part of a ``bias bounties'' program by means of falsification; this is another setting in which models can be made to be contestable.

\section{Basic Settings and Definitions}

We study prediction tasks over a domain $\cZ = \cX\times \cY$. Here $\cX$ represents the \emph{feature} domain and $\cY$ represents the label domain. To avoid dealing with measure-theoretic issues, we assume in this paper that $\cX$ is a discrete set, but this is not essential to any of our results. For this paper we will restrict attention to binary prediction tasks, where $\cY = \{0,1\}$ records the outcome of some binary event. Given a labelled example $(x,y) \in \cZ$, we view $x$ as encoding all observable characteristics of the instance (e.g. meteorological conditions in a weather prediction task, demographic attributes and medical history in a predictive medicine task, etc.), and $y$ represents the binary outcome we are trying to predict (and when part of the training data, represents the outcome of the binary event that we have observed and recorded). 

We model the world via a distribution $\cD \in \Delta \cZ$. 
Generally we will not have a direct description of the distribution, and instead have access only to  a \emph{sample} of $n$ datapoints $D$ sampled i.i.d. from $\cD$, which we will write as $D \in \cZ^n$. We will also sometimes identify a dataset $D = \{(x_1,y_1),\ldots,(x_n,y_n)\}$ with the \emph{empirical distribution} over $D$, which is simply the discrete distribution that places probability mass $1/n$ on each point $(x_i,y_i)$ for $i \in \{1,\ldots,n\}$.

A model is some function $f:\cX\rightarrow [0,1]$, and our (typically unattainable goal) is to find a model $f^*$ that has the property that for all $x \in \cX$, $f^*(x) = \Pr_{(x,y) \sim \cD}[y=1 | x]$ is the \emph{conditional label expectation} given $x$, or (since we are assuming labels are binary) just ``the individual probability'' of the outcome for $x$.  

Suppose someone purports to have a model for individual probabilities $f$. How can we evaluate whether $f$ is any good? If our goal was purely prediction, we might evaluate $f$ via its \emph{squared error} --- i.e. the expected (squared) deviation of its prediction from the true label. This is the objective we would minimize if we were solving (e.g.) a least squares regression problem:

\begin{definition}[Brier Score]
The squared error (also known as \emph{Brier score}) of a model $f$ evaluated on distribution $\cD$ is:
$$B(f,\cD) = \E_{(x,y) \sim \cD}[(f(x) - y)^2]$$

Observe that when we treat a dataset $D = \{(x_1,y_1),\ldots,(x_n,y_n)\}$ as an empirical distribution, then we have:
$$B(f,D) = \frac{1}{n}\sum_{i=1}^n(f(x_i)-y_i)^2$$
\end{definition}

The Brier score can be accurately estimated given access only to samples from a distribution, and a justification for evaluating models via their Brier score is that amongst all models, the Brier score is minimized by the true individual probabilities encoded by a probability distribution.
\begin{lemma}
Fix any probability distribution $\cD$ and let $f^*(x) = \Pr_{(x,y) \sim \cD}[y=1|x]$ represent the true individual probabilities encoded by $\cD$. Let $f:\cX\rightarrow [0,1]$ be any other model. Then:
$$B(f^*,\cD) \leq B(f,\cD)$$
\end{lemma}
Thus if we have two models $f_1$ and $f_2$, and can verify from data that $B(f_1,\cD) < B(f_2,\cD)$, this constitutes an empirical falsification that $f_2$ correctly encodes individual probabilities.

\section{A Reconciliation Procedure}
\label{sec:insample}
Suppose we are given two models $f_1,f_2 :\cX \rightarrow [0,1]$ that purport to predict individual probabilities. Our principle concern is the ``model multiplicity'' problem --- that $f_1$ and $f_2$ differ substantially in their predictions, and yet we cannot falsify either of the two models from the data. Thus we will be interested in regions in which these models disagree substantially in their predictions. We will define ``substantially'' by an arbitrarily small discretization parameter $\epsilon$:
\begin{definition}
Two models $f_1$ and $f_2$ have an $\epsilon$-disagreement on a point $x \in \cX$ if $|f_1(x) - f_2(x)| > \epsilon$. 

Let $U_\epsilon(f_1,f_2)$ be the set of points on which $f_1$ and $f_2$ have an $\epsilon$-disagreement:
$$U_\epsilon(f_1,f_2) = \{x : |f_1(x) - f_2(x)| > \epsilon\}$$
\end{definition}

Informally, we will say that if $f_1$ and $f_2$ do not have an $\epsilon$-disagreement on $x$ that they agree on $x$. 

One way that we can empirically falsify a model $f$ is by measuring the average outcome $y$ on some \emph{subset} or \emph{group} defined on the data, and comparing it to the average prediction of the model $f$ on the same subset. If the two differ substantially, the model must be incorrect. But from finite data we will only be able to accurately measure these averages on groups that are sufficiently large. We will model groups $g$ as indicator functions $g:\cX\rightarrow \{0,1\}$ that specify whether or not each data point $x$ is in the group $(g(x) = 1)$ or not $(g(x) = 0)$. We will use the following notation to measure the size of a group as measured on the underlying distribution $\cD$:

\begin{definition}
Under a distribution $\cD$, a group $g:\cX\rightarrow \{0,1\}$ has probability mass $\mu(g)$ defined as:
$$\mu(g) = \Pr_{(x,y) \sim \cD}[g(x) = 1]$$
\end{definition}

Given a model $f$ and a group $g$, we can define a quantitative extent to which the average prediction of the model on points in $g$ compares to the average (expected) outcome on points in $g$. In expectation over the distribution, these two quantities should agree exactly if $f = f^*$ actually encodes true individual probabilities, but from data we will only be able to estimate these quantities approximately. Thus we define an approximate notion of agreement, which we call approximate group conditional mean consistency. Models $f$ that can be shown not to satisfy approximate group conditional mean consistency on any group $g$ have been falsified, in that this constitutes a proof that $f \neq f^*$ (i.e. $f$ must not encode true individual probabilities). 

\begin{definition}
A model $f:\cX\rightarrow [0,1]$ satisfies $\alpha$-approximate group conditional mean consistency with respect to a group $g \in \cG$ if:
$$\left(\E_{(x,y) \sim \cD}[f(x) | g(x) = 1]-\E_{(x,y) \sim \cD}[y | g(x) = 1]\right)^2 \leq \frac{\alpha}{\mu(g)}$$
\end{definition}

Note that we parameterize $\alpha$-approximate group conditional mean consistency so that it asks for a weaker condition the smaller the size $\mu(g)$ of the group $g$. Informally, for a fixed value of $\alpha$, it asks that the deviation between the average prediction of $f$ and the actual expected outcomes $y$ on a group $g$ differ by at most an error parameter that is proportional to  $1/\sqrt{\mu(g)}$. This will turn out to be the ``right'' scaling because it corresponds to the precision to which we can measure these quantities from data. 

We will show a quantitative version of the following statement. It must be the case that \emph{either}
\begin{enumerate}
    \item $f_1$ and $f_2$ agree on almost all of their predictions, or
    \item $f_1$, or $f_2$, or both can be proven from the data to violate a  group conditional mean consistency condition on a large set of points. In this case, the falsified model can be ``patched'' with a simple update in a way that improves its accuracy. 
\end{enumerate}

The result is that there can be no substantial disagreements about individual probabilities by people who are willing to be convinced by the evidence of the data before them: models which disagree on a substantial fraction of their predictions witness for each other places in which their predictions are falsified by the data, and provide the means to correct (and improve) each other. Thus disagreements can be leveraged to produce improved models, and this process necessarily converges only when the models agree. 

 To formalize this, we start by partitioning the set of $\epsilon$-disagreements $U_\epsilon(f_1,f_2)$ into  two additional  sets that will be important --- the set of disagreements on which $f_1(x) > f_2(x)$, and the set of disagreements on which $f_1(x) < f_2(x)$. 
 
 \begin{definition}
 Fix any two models $f_1,f_2:\cX\rightarrow [0,1]$ and any $\epsilon > 0$. Define the sets:
 $$U_\epsilon^>(f_1,f_2) = \{x \in U_\epsilon(f_1,f_2) : f_1(x) > f_2(x)\}$$
  $$U_\epsilon^<(f_1,f_2) = \{x \in U_\epsilon(f_1,f_2) : f_1(x) < f_2(x)\}$$
  Based on these sets, for $\bullet \in \{>, <\}$ and $i \in \{1,2\}$ define the quantities:
  $$v_*^\bullet = \E_{(x,y) \sim \cD}[y | x \in U_\epsilon^\bullet(f_1,f_2)] \ \ v_i^\bullet = \E_{(x,y) \sim \cD}[f_i(x) | x \in U_\epsilon^\bullet(f_1,f_2)]$$
 \end{definition}
 
 

Our analysis will proceed by showing that if $U_\epsilon(f_1,f_2)$, the set of $\epsilon$-disagreements of $f_1$ and $f_2$ is large, then at least one of the two sets $U_\epsilon^>(f_1,f_2)$ and $U_\epsilon^<(f_1,f_2)$ will witness a large violation of group conditional mean consistency for at least one of the two models. 

\begin{lemma}
\label{lem:reconciler-sets}
Fix any two models $f_1, f_2 : \cX \rightarrow [0,1]$ and any $\epsilon > 0$. 

If the fraction of points on which $f_1$ and $f_2$ have an $\epsilon$ disagreement has mass $\mu(U_\epsilon(f_1,f_2)) = \alpha$ then for some  $\bullet \in \{>, <\}$  some $i \in \{1,2\}$, we have that: 
$$\mu(U_\epsilon^\bullet(f_1,f_2)) \cdot \left(v_*^\bullet- v_i^\bullet \right)^2 \geq \frac{\alpha\epsilon^2}{8}$$

In other words, at least one of the sets $U_\epsilon^>(f_1,f_2)$ and $U_\epsilon^<(f_1,f_2)$ is a group that witnesses an $\frac{\alpha\epsilon^2}{8}$-mean consistency violation for at least one of the models $f_1$ and $f_2$. 
\end{lemma}
\begin{proof}

Since $U_\epsilon(f_1,f_2)$ can be written as the disjoint union:
$$U_\epsilon(f_1,f_2) = U_\epsilon^>(f_1,f_2) \cup U_\epsilon^<(f_1,f_2)$$
we must have that for at least one value of $\bullet \in \{>,<\}$ we have that:
$$\mu(U_\epsilon^\bullet(f_1,f_2)) \geq \frac{\alpha}{2}.$$

Since the points in $U_\epsilon^\bullet(f_1,f_2)$ are $\epsilon$-separated, we must have that $|v_1^\bullet - v_2^\bullet| \geq \epsilon$. Therefore, for at least one of $i \in \{1,2\}$ we must have that $$|v_i^\bullet - v_*^\bullet| \geq \frac{\epsilon}{2}$$
Combining these two claims, we must have that:
$$\mu(U_\epsilon^\bullet(f_1,f_2)) \cdot (v_i^\bullet - v_*^\bullet)^2 \geq \frac{\alpha\epsilon^2}{8}$$
\end{proof}

Let's consider the significance of this Lemma. Most basically, if we have two models $f_1$ and $f_2$ that disagree substantially, this lemma gives an easily constructable set ($U_\epsilon^>(f_1,f_2)$ or $U_\epsilon^<(f_1,f_2)$) that falsifies by a substantial quantitative margin either the assertion that $f_1$ encodes true conditional label expectations or the assertion that $f_2$ does. Next, we show that not only do these sets  falsify that at least one of $f_1$ or $f_2$ are a ``correct'' model --- they provide a directly actionable way to improve one of the models. We prove the following lemma (which is closely related to the kinds of updates used to obtain multicalibrated predictors \citep{hebert2018multicalibration}) which shows us how to improve a model given a group $g$ on which the model fails to satisfy approximate group conditional mean consistency.

\begin{lemma}
\label{lem:reconciler-progress}
Fix any model $f_t :\cX\rightarrow [0,1]$, group $g_t:\cX\rightarrow \{0,1\}$, and distribution $\cD$. Let 
$$\Delta_t = \E_{(x,y) \sim \cD}[y | g_t(x) = 1] -\E_{(x,y) \sim \cD}[f_t(x) | g_t(x) = 1] $$
and 
$$ f_{t+1} = h(x, f_t; g_t,\Delta_t)$$
where $h$ is a ``patch'' defined as:
$$h(x,f; g,\Delta) = \begin{cases}
  f(x) + \Delta  & g(x) = 1 \\
  f(x) & \text{otherwise}
\end{cases}$$
Then:
$$B(f_t,\cD) - B(f_{t+1},\cD) = \mu(g_t)\cdot \Delta_t^2$$
In other words: given any model $f_t$ and a group $g_t$ that witnesses a violation of $\alpha$-approximate group conditional mean consistency on $f_t$, we can efficiently produce a model $f_{t+1}$ that has Brier score that is smaller by exactly $\alpha$. 
\end{lemma}
\begin{proof}
By the definition of the patch  $h(x, f_t; g_t,\Delta_t)$, models $f_t$ and $f_{t+1}$ differ in their predictions only for $x$ such that $g_t(x) = 1$. Therefore we can calculate: 
\begin{eqnarray*}
B(f_t,\cD) - B(f_{t+1},\cD) &=& \Pr[g_t(x) = 0] \cdot \E_{(x,y) \sim \cD}\left[(f_t(x) - y)^2 - (f_{t+1}(x) - y)^2 | g_t(x) = 0 \right] \\ && + \Pr[g_t(x) = 1] \cdot \E_{(x,y) \sim \cD}\left[(f_t(x) - y)^2 - (f_{t+1}(x) - y)^2 | g_t(x) = 1 \right] \\
&=& \mu(g_t) \E_{(x,y) \sim \cD}\left[(f_t(x) - y)^2 - (f_t(x) + \Delta_t - y)^2 | g_t(x) = 1 \right] \\
&=& \mu(g_t) \left(2\Delta_t \E_{(x,y) \sim \cD}\left[y - f_t(x) | g_t(x) = 1\right] - \Delta_t^2 \right) \\
&=& \mu(g_t) \left(2\Delta_t^2 - \Delta_t^2\right) \\
&=& \mu(g_t) \Delta_t^2
\end{eqnarray*}
\end{proof}

Summarizing, whenever we have two models that have $\epsilon$ disagreements on an $\alpha$-fraction of points, we can always constructively falsify at least one of the models, and update it to improve its Brier score by at least $O(\alpha\epsilon^2)$. 

Finally, to make our argument that in-sample quantities (i.e. as measured on the data samples) translate to out of sample quantities (measured on the distribution), it will be useful for our algorithm to not use arbitrarily precise values when patching models, but instead values that are rounded to a finite grid:
\begin{definition}
Fix any integer $m$. 
Let $[1/m]$ denote the set of $m+1$ grid points:
$$\left[\frac{1}{m}\right] = \Bigl\{0, \frac{1}{m}, \frac{2}{m}, \ldots, \frac{m-1}{m}, 1\Bigr\}$$ 
For any value $v \in [0,1]$ let $\text{Round}(v ; m) = \argmin_{v' \in [1/m]} |v-v'|$ denote the closest grid point to $v$ in $[1/m]$.
\end{definition}
Observe that for $v' = \text{Round}(v ; m)$ we always have that $|v - v'| \leq \frac{1}{2m}$

We put this all together in Algorithm \ref{alg:reconciler} (Reconciler). For simplicity of exposition, we initially describe and analyze Algorithm \ref{alg:reconciler} as if it has direct access to distributional quantities. In practice, of course, we will have access only to samples from the distribution, and we will have to run an algorithm on a dataset $D\in \cZ^n$ consisting of these samples. We can do so by interpreting $D$ as the uniform distribution over the samples contained within it. In Section \ref{sec:outofsample}, we show that when we run Reconcile$(f_1,f_2,\alpha,\epsilon,D)$ on a dataset $D \sim \cD^n$ consisting of $n$ i.i.d. samples from $\cD$, then the guarantees we prove in Theorem \ref{thm:reconciler-main} with respect to the empirical distribution $D$ translate over to the true distribution $\cD$ with error terms quickly tending to $0$ as the number of samples $n$ grows large.

\begin{algorithm}[H]
\begin{algorithmic}
\STATE Let $t = t_1 = t_2 = 0$ and $f_1^{t_1} = f_1, f_2^{t_2} = f_2$. 
\STATE Let $m = \lceil \frac{2}{\sqrt{\alpha}\epsilon} \rceil$
\WHILE{$\mu(U_\epsilon(f_1^{t_1},f_2^{t_2})) \geq \alpha$}
  \STATE For each $\bullet \in \{>,<\}$ and $i \in \{1,2\}$ Let:
  $$v_*^\bullet = \E_{(x,y) \sim \cD}[y | x \in U_\epsilon^\bullet(f_1^{t_1},f_2^{t_2})] \ \ v_i^\bullet = \E_{(x,y) \sim \cD}[f_i^{t_i}(x) | x \in U_\epsilon^\bullet(f_1^{t_1},f_2^{t_2})]$$
  \STATE Let:
  $$(i_t,\bullet_t) = \argmax_{i \in \{1,2\}, \bullet \in \{>,<\}} \mu(U_\epsilon^\bullet(f_1^{t_1},f_2^{t_2})) \cdot \left(v_*^\bullet- v_i^\bullet \right)^2 $$
  breaking ties arbitrarily.
  \STATE Let:
  $$g_t(x) = \begin{cases}
  1  & x \in U_\epsilon^{\bullet_t}(f_1^{t_1},f_2^{t_2}) \\
  0 & \text{otherwise}
\end{cases}$$
\STATE Let:
$$\tilde \Delta_t = \E_{(x,y) \sim \cD}[y | g_t(x) = 1] -\E_{(x,y) \sim \cD}[f^{t_{i_t}}_{i_t}(x) | g_t(x) = 1] $$
$$\Delta_t = \text{Round}(\tilde \Delta_t; m)$$
  \STATE Let: $f_i^{t_i+1}(x) = h(x,f_i^{t_i},g_t,\Delta_t)$, $t_i = t_i + 1$, $t = t+1$.
\ENDWHILE
\STATE Output $(f_1^{t_1},f_2^{t_2})$.
\end{algorithmic}
\caption{Reconcile($f_1,f_2,\alpha,\epsilon, \cD$)}
\label{alg:reconciler}
\end{algorithm}

\begin{theorem}
\label{thm:reconciler-main}
For any pair of models $f_1,f_2:\cX\rightarrow [0,1]$, any distribution $\cD$, and any $\alpha, \epsilon > 0$, Algorithm \ref{alg:reconciler} (Reconcile) runs for $T = T_1 + T_2$ many rounds and outputs a pair of models $(f_1^{T_1},f_2^{T_2})$ such that:
\begin{enumerate}
    \item $T \leq (B(f_1,\cD) + B(f_2,\cD))\cdot \frac{16}{\alpha \epsilon^2}$
    \item $B(f_1^{T_1},\cD) \leq B(f_1,\cD) - T_1 \cdot  \frac{\alpha \epsilon^2}{16}$ and $B(f_2^{T_2},\cD) \leq B(f_2,\cD) - T_2 \cdot  \frac{\alpha \epsilon^2}{16}$ 
    \item $\mu(U_\epsilon(f_1^{T_1},f_2^{T_2})) < \alpha$.
\end{enumerate}
\end{theorem}
\begin{remark}
The third conclusion of Theorem \ref{thm:reconciler-main} states that the final models output $(f_1^{T_1},f_2^{T_2})$ approximately agree on their predictions of individual probabilities almost everywhere. The first conclusion states that the reconciliation procedure converges quickly. The second condition of 
Theorem \ref{thm:reconciler-main} focuses on one way that the output models  $(f_1^{T_1},f_2^{T_2})$  are superior to the input models $(f_1,f_2)$ --- they are more accurate. But there is also another way: Every intermediate model $f_1^{t_1}$ and $f_2^{t_2}$ for $t_1 < T_1$ and $t_2 < T_2$ considered by the reconciliation procedure but ultimately not output has been falsified via the demonstration of a set $U_\epsilon^\bullet(\cdot,\cdot)$ on which it fails to satisfy $\alpha$-approximate group conditional mean consistency for a large value of $\alpha$. 
\end{remark}

\begin{proof}
By Lemma \ref{lem:reconciler-sets}, for each round $t < T$ we must have that:
$$\mu(U_\epsilon^{\bullet_t}(f_1^{t_1},f_2^{t_2})) \cdot \left(v_*^{\bullet_t}- v_{i_t}^{\bullet_t} \right)^2 \geq \frac{\alpha\epsilon^2}{8}$$

Let $\tilde f_t^{t_i+1} = h(x,f_i^{t_i},g_t,\tilde \Delta_t)$ --- i.e. the update that would have resulted at round $t$ had the algorithm used the unrounded measurement $\tilde \Delta_t$ rather than the rounded measurement $\Delta_t$. 
By Lemma \ref{lem:reconciler-progress}, we have that:
$$B(f_t^{t_i},\cD) - B(\tilde f_t^{t_i+1},\cD) \geq \frac{\alpha \epsilon^2}{8}$$.

We can now compute 
\begin{eqnarray*}
B(f_t^{t_i},\cD) - B(f_t^{t_i+1},\cD) &=& (B(f_t^{t_i},\cD) - B(\tilde f_t^{t_i+1},\cD)) - ( B(f_t^{t_i+1},\cD) -  B(\tilde f_t^{t_i+1},\cD)) \\
&\geq& \frac{\alpha\epsilon^2}{8}- ( B(f_t^{t_i+1},\cD) -  B(\tilde f_t^{t_i+1},\cD))
\end{eqnarray*}
So it remains to upper bound $( B(f_t^{t_i+1}) -  B(\tilde f_t^{t_i+1}))$. Let $\hat \Delta = \tilde \Delta_t - \Delta_t$. We make several observations: First, $\tilde f_t^{t_i+1} = h(x,f_t^{t_i+1},g_t,\hat \Delta)$. Second, 
\begin{eqnarray*}
\hat \Delta &=&   \E_{(x,y) \sim \cD}[y | g_t(x) = 1] -\E_{(x,y) \sim \cD}[f^{t_i}_i(x) | g_t(x) = 1] - \Delta_t\\
&=&  \E_{(x,y) \sim \cD}[y | g_t(x) = 1] -\E_{(x,y) \sim \cD}[f^{t_i+1}_i(x) | g_t(x) = 1]
\end{eqnarray*}
Third, by definition of the Round operation, $|\hat \Delta| \leq \frac{1}{2m}$. Therefore we can again apply Lemma \ref{lem:reconciler-progress} to conclude that:
\begin{eqnarray*}
B(f_t^{t_i+1},\cD) -  B(\tilde f_t^{t_i+1},\cD) &=& \mu(g_t) \hat \Delta^2 \\
&\leq& \frac{1}{4m^2}
\end{eqnarray*}

Combining this with our initial calculation lets us conclude that:
$$B(f_t^{t_i},\cD) - B(f_t^{t_i+1},\cD) \geq \frac{\alpha\epsilon^2}{8} - \frac{1}{4m^2} \geq \frac{\alpha\epsilon^2}{16}$$
Here we are using the fact that we have set $m \geq \frac{2}{\sqrt{\alpha}\epsilon}$. 
Applying this lemma for each of the $T_1$ and $T_2$ updates for $f_1$ and $f_2$, respectively, we get that: $B(f_1^{T_1},\cD) \leq B(f_1,\cD) - T_1 \cdot  \frac{\alpha \epsilon^2}{16}$ and $B(f_2^{T_2},\cD) \leq B(f_2,\cD) - T_2 \cdot  \frac{\alpha \epsilon^2}{16}$. Since Brier scores are non-negative, we conclude that $T_1 \leq B(f_1,\cD) \frac{16}{\alpha \epsilon^2}$ and $T_2 \leq B(f_2,\cD)  \frac{16}{\alpha \epsilon^2}$. Thus $T = T_1+T_2 \leq  (B(f_1,\cD) + B(f_2,\cD))\cdot \frac{16}{\alpha \epsilon^2}$

Finally the halting condition of the algorithm implies that $\mu(U_\epsilon(f_1^{T_1},f_2^{T_2})) < \alpha$.
\end{proof}
Thus if we start with any two models that have substantial disagreement, we are guaranteed to be able to efficiently produce \emph{strictly improved} models that almost agree almost everywhere. In particular, we can never be in a position in which we have two equally accurate \emph{but unimprovable} models that have substantial disagreements: in this case, we can always improve the models. The only time we can have substantial model disagreement is if we refuse to improve the models even in the face of efficiently verifiable and actionable evidence that one of the models is suboptimal and improvable. 

We observe that any pair of models that have gone through the ``Reconcile'' process must also produce very similar estimates for the conditional label expectation over any sufficiently large reference class (i.e. any subset of the feature space). In particular, for any sufficiently large reference class, either both models are consistent with the data or they are not --- but they cannot substantially disagree.

\begin{corollary}
Let $E \subset \cX$ be any subset of the feature space. Let $f_1$ and $f_2$ be any two models that have been output by Algorithm \ref{alg:reconciler} (Reconcile) with parameters $\epsilon$ and $\alpha$. Let:
$$p_1(E) = \E_{(x,y) \sim \cD}[f_1(x) | x \in E] \text{  and  } p_2(E) = \E_{(x,y) \sim \cD}[f_2(x) | x \in E]$$
be the estimates for $\Pr[y=1 | x \in E]$ implied by models $f_1$ and $f_2$ respectively. Then:
$$|p_1(E) - p_2(E)| \leq \frac{\alpha}{\mu(E)}+\epsilon$$
\end{corollary}

\begin{proof}
Let $S_\epsilon(f_1,f_2) = \{x : x \not\in U_\epsilon(f_1,f_2)\}$ be the set of points on which $f_1$ and $f_2$ do not have an $\epsilon$-disagreement. Recall that $\mu(S_\epsilon(f_1,f_2)) \geq 1-\alpha$.
We compute:
\begin{eqnarray*}
\mu(E)|p_1(E) - p_2(E)| &=& \left|\sum_{x \in E} \mu(\{x\}) \cdot (f_1(x) - f_2(x))\right| \\
&=& \left| \sum_{x \in E \cap U_\epsilon(f_1,f_2)} \mu(\{x\})\cdot (f_1(x) - f_2(x)) + \sum_{x \in E \cap S_\epsilon(f_1,f_2)} \mu(\{x\})\cdot (f_1(x) - f_2(x))\right| \\
&\leq& \alpha + \mu( E \cap S_\epsilon(f_1,f_2)) \epsilon \\
&\leq& \alpha + \mu(E)\epsilon
\end{eqnarray*}
Dividing by $\mu(E)$ yields the corollary. 
\end{proof}

\section{Data  Requirements}
\label{sec:outofsample}
We have presented our Algorithm \ref{alg:reconciler} as if it has direct access to the distribution $\cD$. Of course in general we do not have access to $\cD$, but rather have access to some set $D$ of $n$ i.i.d. \emph{samples} from $\cD$. We will typically instead run Algorithm \ref{alg:reconciler} over the \emph{empirical distribution} over $D$. We will prove that with high probability over the sample of $D$, when Algorithm \ref{alg:reconciler} is run over the empirical distribution on $D$, then its guarantees translate over to the distribution $\cD$ from which $D$ was drawn, with error parameters that go to zero with the size $n$ of the data sample. It will be important that the dataset $D$ that we run Algorithm \ref{alg:reconciler} on is sampled \emph{independently} of the models $f_1$ and $f_2$ to be reconciled --- i.e. it should be fresh data that was not used to train either of the models being reconciled. 

We begin by counting the number of potential models $f_1^{t_1},f_2^{t_2}$ that Algorithm \ref{alg:reconciler} might output.

\begin{lemma}
\label{lem:reconciler-model-counting}
Fix any pair of models $f_1,f_2:\cX\rightarrow [0,1]$ and any $\alpha,\epsilon > 0$. Then there is a set $C$ of pairs of models of size at most $|C| \leq \left(4\cdot(m+1)\right)^{32/\alpha\epsilon^2+1}$ such that for any distribution $\cD$ on which Algorithm \ref{alg:reconciler} is run, the output models $(f_1^{t_1},f_2^{t_2}) \in C$. Here, as in Algorithm \ref{alg:reconciler}, $m = \lceil\frac{2}{\sqrt{\alpha}\epsilon} \rceil$.
\end{lemma}
\begin{proof}
Given a run of Algorithm \ref{alg:reconciler} for $T$ rounds, let $\pi = \{(i_t,\bullet_t,\Delta_t)\}_{t=1}^T$ denote the record of the quantities $(i_t,\bullet_t,\Delta_t)$ chosen at each round $t$. Let $\pi^{<t} =\{(i_{t'},\bullet_{t'},\Delta_{t'})\}_{t'=1}^{t-1}$ denote the prefix of this transcript up through round $t-1$. Observe that once we fix $\pi^{<t}$ we have also fixed the models $f_1^{t_1}$ and $f_2^{t_2}$ that are defined at the start of round $t$. To see this, assume the claim holds true at round $t$. In particular, $\pi^{<t}$ fixes the disagreement regions $U_\epsilon^{\bullet}(f_1^{t_1},f_2^{t_1})$ of these two models, and therefore given the choices $(i_t,\bullet_t,\Delta_t)$, we have inductively defined the models present at the start of round $t+1$. 

We let $C$ denote the set of all pairs of models defined by transcripts $\pi^{<T}$ for all $T \leq \frac{32}{\alpha\epsilon^2}$. Since we know from Theorem \ref{thm:reconciler-main} that Algorithm \ref{alg:reconciler} halts after at most $T \leq  (B(f_1,\cD) + B(f_2,\cD))\cdot \frac{16}{\alpha \epsilon^2} \leq \frac{32}{\alpha\epsilon^2}$ many rounds, the models output by Algorithm \ref{alg:reconciler} must be contained in $C$ as claimed. It remains to count the set of transcripts of length $T \leq \frac{32}{\alpha\epsilon^2}$. At each round $t$, there are two possible values for $i_t$, two possible values for $\bullet_t$, and $m+1$ possible choices for $\Delta_t$. Hence the number of transcripts of length $T$ is $(4(m+1))^T$. Thus we have:
$$|C| \leq \sum_{T=0}^{\frac{32}{\alpha\epsilon^2}}(4(m+1))^T \leq (4(m+1))^{\frac{32}{\alpha\epsilon^2}+1}$$
\end{proof}

We can now argue that if we have a sample of $n$ datapoints $D$ that are sampled i.i.d. from some unknown distribution $\cD$, then if we run Algorithm \ref{alg:reconciler} using the empirical distribution over $D$, then its guarantees hold also over $\cD$, with error terms that tend to $0$ as $n$ grows large. 

\begin{restatable}{theorem}{outofsample}
\label{thm:reconciler-out-of-sample}
Fix any data distribution $\cD$ and consider a run of Algorithm \ref{alg:reconciler} over the empirical distribution over points in a dataset $D \sim \cD^n$ consisting of $n$ points sampled i.i.d. from $\cD$.  For any pair of models $f_1,f_2:\cX\rightarrow [0,1]$ and any $\alpha, \epsilon > 0$, Algorithm \ref{alg:reconciler} (Reconcile) runs for $T = T_1 + T_2$ many rounds and outputs a pair of models $(f_1^{T_1},f_2^{T_2})$ such that:
\begin{enumerate}
    \item $T \leq \frac{16}{\alpha \epsilon^2}$
    \item For any $\delta > 0$, with probability at least $1-\delta$ over the randomness of $D \sim \cD^n$ we have that: 
$$B(f_1^{T_1},\cD) \leq B(f_1,\cD) - T_1 \cdot  \frac{\alpha \epsilon^2}{16} + 2\sqrt{\frac{\left(\frac{16}{\alpha\epsilon^2}+1\right)\log\left(\frac{64\left(\lceil \frac{2}{\sqrt{\alpha}\epsilon}\rceil+1\right)}{\delta}\right)}{n}}$$ and
$$B(f_2^{T_2},\cD) \leq B(f_2,\cD) - T_2 \cdot  \frac{\alpha \epsilon^2}{16} + 2\sqrt{\frac{\left(\frac{16}{\alpha\epsilon^2}+1\right)\log\left(\frac{64\left(\lceil \frac{2}{\sqrt{\alpha}\epsilon}\rceil+1\right)}{\delta}\right)}{n}}$$
    \item For any $\delta > 0$, with probability at least $1-\delta$ over the randomness of $D \sim \cD^n$: $$\mu(U_\epsilon(f_1^{T_1},f_2^{T_2})) < \alpha + \sqrt{\frac{\left(\frac{32}{\alpha\epsilon^2}+1\right)\log\left(\frac{8\left(\lceil \frac{2}{\sqrt{\alpha}\epsilon}\rceil+1\right)}{\delta}\right)}{n}}.$$
\end{enumerate}
\end{restatable}

\begin{remark}
Theorem \ref{thm:reconciler-out-of-sample} tells us that the guarantees we proved for Algorithm \ref{alg:reconciler} in Theorem \ref{thm:reconciler-main} (when we assumed direct access to the distribution $\cD$) continue to hold when all we have access to is a finite sample of $n$ points from the data distribution, with additional error terms that tend to zero as $n$ grows large. How large is large? If we want the final disagreement region to have mass at most $2\alpha$ (i.e. we want the third conclusion of Theorem \ref{thm:reconciler-out-of-sample} to tell us that $\mu(U_\epsilon(f_1^{T_1},f_2^{T_2})) < 2\alpha$), then solving for $n$ in the error bound, we find that it suffices to have $n$ samples for $n$ on the order of:
$$n \in \tilde O \left(\frac{\log(1/\delta)}{\alpha^3\epsilon^2} \right)$$
where the $\tilde O()$ notation hides logarithmic terms in $1/\alpha$ and $1/\epsilon$. 

This is a remarkably small amount of data: Recall that we would need $\approx \frac{\log (1/\delta)}{\alpha \epsilon^2}$ samples just to estimate the conditional label expectation $\Pr[y = 1 | x \in S]$ for a conditional event $S$ with $\mu(S) = \alpha$ up to error $\epsilon$ with probability $1-\delta$ (or for two parties with disjoint samples to agree on this conditional label expectation up to error $\epsilon$). Theorem \ref{thm:reconciler-out-of-sample} tells us that in fact two parties can be made to agree on a $1-\alpha$ fraction of points up to error $\epsilon$ with an additional amount of data only on the order of $\tilde O(1/\alpha^2)$. Crucially this bound is \emph{independent} of the complexity of the models $f_1$ and $f_2$, and so we can think of the initial models as being arbitrary (and arbitrarily sophisticated). 
\end{remark}

\begin{remark}
Rather than updating on the disagreement regions $U^\bullet_\epsilon (f_1,f_2)$, we could update separately on regions on which the two models predict values $f_1(x) = v$ and $f_2(x) = v'$ for each $v \neq v'$, to produce a model that is \emph{cross calibrated} in the sense of \cite{feinberg2008testing}. \cite{garg2019tracking} give an algorithm for doing this called ``Merge'' that we could use for our purposes with small modifications\footnote{We would have to bucket both predictor's outputs into $O(1/\epsilon)$ many buckets,  ignore pairs of predictions $(v, v')$ that have probability mass smaller than $\alpha\epsilon^2$, and produce two models, one for each party, that might disagree on regions in which the original models did not have $\epsilon$-disagreements.}. Carrying this through would yield an algorithm that would require $n = \tilde O(\log(1/\delta)/\alpha\epsilon^4)$ many samples, which is incomparable to the guarantee we obtain in Theorem \ref{thm:reconciler-out-of-sample} for Algorithm \ref{alg:reconciler}.
\end{remark}
The proof of Theorem \ref{thm:reconciler-out-of-sample}---which is in Appendix \ref{app:outofsample}---is a straightforward argument combining Hoeffding's inequality (Theorem \ref{thm:hoeffding}) with a union bound over the models enumerated in Lemma \ref{lem:reconciler-model-counting}. Informally, we will argue that for each of the pairs of models $(f_1,f_2)$ in $C$ enumerated in Lemma \ref{lem:reconciler-model-counting}, their Brier scores are similar as evaluated over $\cD$ and $D$, and their disagreement regions $U_\epsilon(f_1,f_2)$ have similar mass over $\cD$ and $D$. For any fixed pair of models $(f_1,f_2)$ these statements follow from Hoeffding's inequality --- and the fact that $C$ has bounded cardinality means that we can at small cost in data obtain these guarantees uniformly over all pairs of models in $C$. 

\section{Contestable Models}
\label{sec:contestable}
Thus far we have considered the problem of reconciling \emph{two} models $f_1$ and $f_2$, and have shown that we require only $O(1/(\alpha^3 \epsilon^2))$ many points to obtain strictly improved models $f_1', f_2'$ that have $\epsilon$ disagreements on at most an $\alpha$ measure of points. But what if someone then proposes a third model, $f_3$, and then another $f_4$, etc? We could run the reconciliation process again each time---and perhaps if we had $k$ models, repeatedly in a pairwise fashion until all $k$ of the models approximately agreed---but this would naively require a fresh set of samples for each new reconciliation procedure. In this section, we show how to do better: we attach to $f$ just a single sample of ``contestation'' data that is of size  polynomial in our target reconciliation parameters $\alpha$ and $\epsilon$ (and independent of the complexity of the model or distribution). Using this data, we show that we  can then put $f$ through a reconciliation procedure with a very large (exponential in the size of its contestation data set) number of models, with the same guarantees as if we had run the models through Algorithm \ref{alg:reconciler} each time. Driving this result is the observation that each time a particular model $f$ is updated using the patch operation defined in Lemma \ref{lem:reconciler-progress}, $f$'s squared error drops, independently of which reconciliation process the update is a part of --- and thus the total number of large updates made to a single model is bounded independently of the number of other models that are ``reconciled'' with it. This, together with results from adaptive data analysis that allow us to repeatedly re-use hold-out sets while preserving statistical validity \citep{dwork2015generalization,bassily2021algorithmic,jung2020new} are enough to give the result.

We define a ``contestable model'' to be a model $f$ attached to a fixed sample of ``contestation data''. A ``contestable'' model can be ``contested'' by identifying any subset of the data identified by an indicator function $g:\cX\rightarrow \{0,1\}$. The guarantee of a contestable model is that if it is contested using a subset of the data that is simultaneously \emph{large} and on which the expectation of the model's predictions is substantially different than the expectation of the label, then the model will be updated in a way that corrects the discovered error on the identified subset of data, and strictly improves the squared error of the model. These contestations are \emph{accepted}. Contestations can also be \emph{rejected} on the grounds either that the group identified is too small, or that the model already predicts on average a value over that group that is sufficiently close to the true label mean over that group. We aim to design contestable models that can receive a number of contestations over their lifetime that is exponential in the size of their contestation dataset. 

\begin{definition}
A \emph{contestable model} $\mathbf{f}$ consists of a current model $f_c:\cX\rightarrow [0,1]$, a dataset $D \in \cZ^n$,  and has two operations: $\mathbf{f}.\mathrm{predict}(x)$ which takes as input a data point $x \in \cX$ and $\mathbf{f}.\mathrm{contest}(g)$ which takes as input the indicator function for a group $g:\cX\rightarrow \{0,1\}$. $\mathbf{f}.\mathrm{predict}(x)$ outputs $f_c(x)$, where $f_c$ is the current model belonging to $\mathbf{f}$, and $\mathbf{f}.\mathrm{contest}(g)$ may update the current model $f_c$ to a new model $f_{c+1}$ according to Algorithm \ref{alg:contest}.
\end{definition}

Algorithm \ref{alg:contest}, which follows, is a randomized algorithm: it samples from the \emph{Laplace} distribution. We write Lap($b$) to denote the sampling operation for the centered Laplace distribution with scale parameter $b$, which is the distribution that has probability density function $f(x; b) = \frac{1}{2b}\exp\left(-\frac{|x|}{b} \right)$.
\begin{algorithm}[]
\begin{algorithmic}
\STATE Given: A failure probability $\delta$, a dataset $D \in \cZ^n$, an initial model $f = f_0$, a threshold $T$ to accept attempted contestations, a target total number of contestations $K$, and an upper bound $C$ on the total number of accepted contestations.
\STATE Let $t = 0$ denote a count of the number of attempted contestations and $c = 0$ denote the count of the number of accepted contestations. 
\STATE Let privacy parameter $\epsilon = \sqrt{\frac{\log \frac{K}{\delta}\sqrt{C\ln\frac{1}{\delta}}}{ n}}$
\STATE Let $\epsilon_1 = \frac{\sqrt{512}}{\sqrt{512}+1}\epsilon$, $\epsilon_2 = \frac{1}{\sqrt{512}+1}\epsilon$
\STATE Let $\sigma(\epsilon) = \frac{\sqrt{32 C \log(1/\delta)}}{\epsilon n}$
\STATE Let $\hat T_0 = T + \textrm{Lap}(\sigma(\epsilon_1))$
\WHILE{there is another model $g_t$ given as input to $\textbf{f}$.contest$(g_t)$ and $c < C$}
  \STATE Compute an empirical estimate of $\mu(g_t)\cdot \E[y - f_c(x)|g_t(x) = 1]$: 
  $$\eta_t(f_c,g_t) =\frac{1}{n}\sum_{(x,y) \in D} (y-f_c(x))\cdot g_t(x)$$
 \STATE Let $\hat \eta_t = |\eta_t(f_c,g_t)| + \textrm{Lap}\left(2\sigma(\epsilon_1)\right)$
 \IF{$\hat \eta_t \geq \hat T_c$}
   \STATE The contestation is \emph{accepted}.
   \STATE Let: 
   $$\tilde \mu_t = \frac{1}{n}\sum_{(x,y) \in D}g_t(x) + \textrm{Lap}(2\sigma(\epsilon_2)) \ \ \ \tilde \eta_t = \eta_t(f_c,g_t) + \textrm{Lap}(2\sigma(\epsilon_2)) $$
   \STATE Let $\tilde \Delta_t = \frac{\tilde \eta_t}{\tilde \mu_t}$
    \STATE Let $f_{c+1}(x) = h(x,f_c,g_t,\tilde \Delta_t)$, $c = c+1$.
   \STATE Let $\hat T_c = T + \textrm{Lap}(\sigma(\epsilon_1))$
 \ELSE
   \STATE The contestation is \emph{rejected}.
 \ENDIF
  \STATE Let $t = t+1$.
\ENDWHILE
\STATE Halt.
\end{algorithmic}
\caption{$\mathbf{f}$.Contest}
\label{alg:contest}
\end{algorithm}

The analysis of Algorithm \ref{alg:contest} is in Appendix \ref{app:contestable} and goes through \emph{differential privacy} \citep{dwork2006calibrating}. Differential privacy was originally introduced as a strong notion of privacy that could be satisfied while still carrying out high accuracy statistical analyses, but has since found many other uses. Our interest in differential privacy will be because of the transfer theorems of \cite{dwork2015generalization,bassily2021algorithmic,jung2020new} which informally state that analyses that are both differentially private and accurate on a sample of data $D$ drawn i.i.d. from an underlying distribution $\cD$ must also be accurate on the underlying distribution. Algorithm \ref{alg:contest}  is an instantiation of Algorithm 3 (NumericSparse) from \cite{dwork2014algorithmic}, from which it follows that the algorithm is differentially private in the contestation dataset $D$. We then apply the version of the ``transfer theorem'' given in \cite{jung2020new}, which establishes that its estimates of statistics on $D$ are representative of their true values on the underlying distribution $\cD$ from which $D$ was drawn. Our use of differential privacy here to get out of sample guarantees closely mirrors its use in \cite{hebert2018multicalibration} to get a generalization theorem for multicalibration algorithms. In fact, if the ``contestation'' sets $g_t$ submitted to $\mathbf{f}$.Contest$(g_t)$ were the groups with respect to which which multicalibrated predictors are required to satisfy group conditional mean consistency, then Algorithm \ref{alg:contest} would essentially (up to some details) be the multicalibration algorithm originally given by \cite{hebert2018multicalibration}. 
But a contestable model can take as input the indicator function $g_t$ of \emph{any} group, including those groups $U_\epsilon^\bullet(f_1^{t_1},f_2^{t_2}) $ used as updates within Reconcile (Algorithm \ref{alg:reconciler}). Thus, contestable models will be able to be reconciled with many other models in a data efficient way (in addition to being ``contested'' on other groups on which they fail to satisfy group conditional mean consistency).

\begin{restatable}{theorem}{contestable}
\label{thm:contestable}
Initialized with a dataset $D \sim \cD^n$ of size $n$ sampled i.i.d. from $\cD$, a target number of contestations $K$, a failure probability $\delta$, a threshold $T = \Theta\left( \frac{\left(\log \frac{K}{\delta}\right)^{1/3}\left( \ln \frac{1}{\delta}\right)^{1/6}}{n^{1/3}}\right)$ and a limit on successful contestations $C = \Theta \left(\frac{1}{T^2} \right)$, a contestable model will with probability $1-2\delta n$: 
\begin{enumerate}
    \item Process at least $K$ contestations $g_t$ without halting,
    \item Guarantee that every accepted contestation $g_t$ is such that:
    $$\left|\E_{(x,y) \sim \cD}[g_t(x)(y-f_c(x))] \right|\geq \Omega\left( \frac{\left(\log \frac{K}{\delta}\right)^{1/3}\left( \ln \frac{1}{\delta}\right)^{1/6}}{n^{1/3}}\right)$$ and produces an update that reduces the squared error of $\mathbf{f}$ by $$B(f_{c}) - B(f_{c+1}) = \Omega(T^2) = \Omega\left( \frac{\left(\log \frac{K}{\delta}\right)^{2/3}\left( \ln \frac{1}{\delta}\right)^{1/3}}{n^{2/3}}\right) $$
    \item Guarantee that every rejected contestation $g_t$ is such that:
      $$\left|\E_{(x,y) \sim \cD}[g_t(x)(y-f_c(x))]\right| \leq O\left( \frac{\left(\log \frac{K}{\delta}\right)^{1/3}\left( \ln \frac{1}{\delta}\right)^{1/6}}{n^{1/3}}\right)$$
\end{enumerate}
\end{restatable}
The proof of Theorem \ref{thm:contestable} can be found in Appendix \ref{app:contestable}. 

We now observe how a contestable model with the guarantees of Theorem \ref{thm:contestable} can be repeatedly used as part of a reconciliation procedure akin to Algorithm \ref{alg:reconciler}.

\begin{algorithm}[H]
\begin{algorithmic}
\WHILE{$\mu(U_\epsilon(\mathbf{f}_1,\mathbf{f}_2)) \geq \alpha$}
  \STATE For $\bullet \in \{>, <\}$ let:
   $$g^\bullet(x) = \begin{cases}
  1  & x \in U_\epsilon^{\bullet}(\mathbf{f}_1,\mathbf{f}_2) \\
  0 & \text{otherwise}
\end{cases}$$
  \STATE $\textbf{f}_1.contest(g^>)$,  $\textbf{f}_1.contest(g^<)$, $\textbf{f}_2.contest(g^>)$,  $\textbf{f}_2.contest(g^<)$
\ENDWHILE
\end{algorithmic}
\caption{Contestable-Reconcile($\mathbf{f}_1,\mathbf{f}_2,\alpha,\epsilon, \cD$)}
\label{alg:contestable-reconciler}
\end{algorithm}

 The idea is simple (and outlined in Algorithm \ref{alg:contestable-reconciler}): While we have two contestable models $\mathbf{f}_1$ and $\mathbf{f}_2$ that have $\epsilon$-disagreements on more than an $\alpha$ fraction of the distribution, contest both models on the disagreement sets $U_\epsilon^{>}(\mathbf{f}_1,\mathbf{f}_2)$ and $U_\epsilon^{<}(\mathbf{f}_1,\mathbf{f}_2)$. By Lemma \ref{lem:reconciler-sets}, if indeed $\mu(U_\epsilon^{\bullet}(\mathbf{f}_1,\mathbf{f}_2))\geq \alpha$, then for at least one of the models $i \in \{1,2\}$ and for at least one of the sets $\bullet \in \{>,<\}$, we must have:
$$\left|\mu(U_\epsilon^\bullet(f_1,f_2)) \cdot \E[y - f_c(x)|x \in U_\epsilon^\bullet(f_1,f_2)]\right| \geq \mu(U_\epsilon^\bullet(f_1,f_2)) \cdot \E[y - f_c(x)|x \in U_\epsilon^\bullet(f_1,f_2)]^2 \geq \frac{\alpha\epsilon^2}{8}$$

By Theorem \ref{thm:contestable}, assuming that the contestation datasets of both models are of size:
$$n \geq \Omega \left(\frac{\log \frac{K}{\delta}\sqrt{\log \frac{1}{\delta}}}{\alpha^3\epsilon^6} \right)$$
then at least one of these contestations will succeed until (after at most a polynomial number of contestations in $\alpha$ and $\epsilon$), the models are reconciled and $\mu(U_\epsilon^{\bullet}(\mathbf{f}_1,\mathbf{f}_2))\leq \alpha$.

Solving for $K$, we find that a contestable model can be run through 
$$K = \tilde \Theta \left(\delta \exp\left(\frac{n\alpha^3\epsilon^6}{\sqrt{\log\frac{1}{\delta}}} \right)\right)$$
many contestation procedures given a single contestation dataset of size $n$. Here the $\tilde \Theta$ hides terms that are logarithmic in $1/\alpha$ and $1/\epsilon$. 

We emphasize that a contestable model can be contested using $K$ many sets of any nature --- these can include the disagreement regions that arise from our Reconcile procedure, but can also include arbitary regions on which the current model is found to be mis-calibrated. Thus contestable models can be robustly and iteratively improved over an exponential number of contestations whenever they are falsified by being shown to fail to satisfy group conditional mean consistency on any group. Modest amounts of data, attached to a model as a contestation dataset, can make the model long-lived in an  easily adaptable and improvable form.

\section{Conclusion}

Individual probability assignments are not determined by data; this lies at the heart of both the reference class problem and the predictive multiplicity problem. Insofar as individual probability assignments play a significant role in consequential decision-making, their underdetermination by data may give rise to practical problems when we have two or more seemingly equally good estimation methods that nevertheless result in models that differ substantially in the assignments they predict. We show that given modest amounts of data to resolve disagreements, such problems cannot arise at a substantial scale, because if two models disagree substantially in many places, then this large disagreement region itself points us to how to improve at least one of the models. 
The only way this process can conclude is with improved models that approximately agree almost everywhere. This does not ``resolve'' the reference class problem, the predictive multiplicity problem, or other puzzles about individual probability in that it does not claim a way to produce ``correct'' estimates of individual probabilities. But it does remove the practical bite of these problems in that it shows that two parties who agree on the data distribution and who have committed in good faith to make statistical estimates of individual probabilities cannot end up in a state where they substantially disagree on a large number of instances --- and hence will rarely face any ambiguity in how they should act, given their statistical modeling.

\subsection*{Acknowledgements}
We thank Philip Dawid, Konstantin Genin, Benjamin Jantzen, Leonard Smith, Kareem Khalifa, Omer Reingold, Alvin Roth, and Rakesh Vohra for helpful comments and making connections to the literature. A.R. was supported in part by NSF grants FAI-2147212 and CCF-2217062 and the Simons Collaboration on the Theory of Algorithmic Fairness.

\bibliographystyle{plainnat}
\bibliography{ref}

\begin{thebibliography}{39}
\providecommand{\natexlab}[1]{#1}
\providecommand{\url}[1]{\texttt{#1}}
\expandafter\ifx\csname urlstyle\endcsname\relax
  \providecommand{\doi}[1]{doi: #1}\else
  \providecommand{\doi}{doi: \begingroup \urlstyle{rm}\Url}\fi

\bibitem[Aaronson(2005)]{aaronson2005complexity}
Scott Aaronson.
\newblock The complexity of agreement.
\newblock In \emph{Proceedings of the thirty-seventh annual ACM symposium on
  Theory of computing}, pages 634--643, 2005.

\bibitem[Al-Najjar and Weinstein(2008)]{al2008comparative}
Nabil~I Al-Najjar and Jonathan Weinstein.
\newblock Comparative testing of experts.
\newblock \emph{Econometrica}, 76\penalty0 (3):\penalty0 541--559, 2008.

\bibitem[Aumann(1976)]{aumann1976agreeing}
Robert~J Aumann.
\newblock Agreeing to disagree.
\newblock \emph{The Annals of Statistics}, 4\penalty0 (6):\penalty0 1236--1239,
  1976.

\bibitem[Barda et~al.(2020)Barda, Riesel, Akriv, Levy, Finkel, Yona, Greenfeld,
  Sheiba, Somer, Bachmat, et~al.]{barda2020developing}
Noam Barda, Dan Riesel, Amichay Akriv, Joseph Levy, Uriah Finkel, Gal Yona,
  Daniel Greenfeld, Shimon Sheiba, Jonathan Somer, Eitan Bachmat, et~al.
\newblock Developing a covid-19 mortality risk prediction model when
  individual-level data are not available.
\newblock \emph{Nature communications}, 11\penalty0 (1):\penalty0 1--9, 2020.

\bibitem[Barda et~al.(2021)Barda, Yona, Rothblum, Greenland, Leibowitz,
  Balicer, Bachmat, and Dagan]{barda2021addressing}
Noam Barda, Gal Yona, Guy~N Rothblum, Philip Greenland, Morton Leibowitz, Ran
  Balicer, Eitan Bachmat, and Noa Dagan.
\newblock Addressing bias in prediction models by improving subpopulation
  calibration.
\newblock \emph{Journal of the American Medical Informatics Association},
  28\penalty0 (3):\penalty0 549--558, 2021.

\bibitem[Bassily et~al.(2021)Bassily, Nissim, Smith, Steinke, Stemmer, and
  Ullman]{bassily2021algorithmic}
Raef Bassily, Kobbi Nissim, Adam Smith, Thomas Steinke, Uri Stemmer, and
  Jonathan Ullman.
\newblock Algorithmic stability for adaptive data analysis.
\newblock \emph{SIAM Journal on Computing}, 50\penalty0 (3):\penalty0
  STOC16--377, 2021.

\bibitem[Bastani et~al.(2022)Bastani, Gupta, Jung, Noarov, Ramalingam, and
  Roth]{bastani2022practical}
Osbert Bastani, Varun Gupta, Christopher Jung, Georgy Noarov, Ramya Ramalingam,
  and Aaron Roth.
\newblock Practical adversarial multivalid conformal prediction.
\newblock \emph{arXiv preprint arXiv:2206.01067}, 2022.

\bibitem[Bienvenu et~al.(2011)Bienvenu, G{\'{a}}cs, Hoyrup, Rojas, and
  Shen]{bienvenuetal}
Laurent Bienvenu, Peter G{\'{a}}cs, Mathieu Hoyrup, Cristobal Rojas, and
  Alexander Shen.
\newblock Algorithmic tests and randomness with respect to a class of measures.
\newblock \emph{Proc. Steklov Inst. Math.}, 274:\penalty0 34--89, 2011.

\bibitem[Black et~al.(2022)Black, Raghavan, and Barocas]{black2022multiplicity}
Emily Black, Manish Raghavan, and Solon Barocas.
\newblock Model multiplicity: Opportunities, concerns, and solutions.
\newblock In \emph{2022 ACM Conference on Fairness, Accountability, and
  Transparency}, FAccT '22, page 850–863, New York, NY, USA, 2022.
  Association for Computing Machinery.
\newblock ISBN 9781450393522.
\newblock \doi{10.1145/3531146.3533149}.
\newblock URL \url{https://doi.org/10.1145/3531146.3533149}.

\bibitem[Breiman(2001)]{breiman2001statistical}
Leo Breiman.
\newblock Statistical modeling: The two cultures (with comments and a rejoinder
  by the author).
\newblock \emph{Statistical science}, 16\penalty0 (3):\penalty0 199--231, 2001.

\bibitem[D'Amour et~al.(2020)D'Amour, Heller, Moldovan, Adlam, Alipanahi,
  Beutel, Chen, Deaton, Eisenstein, Hoffman, et~al.]{d2020underspecification}
Alexander D'Amour, Katherine Heller, Dan Moldovan, Ben Adlam, Babak Alipanahi,
  Alex Beutel, Christina Chen, Jonathan Deaton, Jacob Eisenstein, Matthew~D
  Hoffman, et~al.
\newblock Underspecification presents challenges for credibility in modern
  machine learning.
\newblock \emph{arXiv preprint arXiv:2011.03395}, 2020.

\bibitem[Dawid(1985)]{dawid85}
A.~P. Dawid.
\newblock {Calibration-Based Empirical Probability}.
\newblock \emph{The Annals of Statistics}, 13\penalty0 (4):\penalty0 1251 --
  1274, 1985.
\newblock \doi{10.1214/aos/1176349736}.
\newblock URL \url{https://doi.org/10.1214/aos/1176349736}.

\bibitem[Dawid(1982)]{dawid1982well}
A~Philip Dawid.
\newblock The well-calibrated bayesian.
\newblock \emph{Journal of the American Statistical Association}, 77\penalty0
  (379):\penalty0 605--610, 1982.

\bibitem[Dawid(2017)]{dawid2017individual}
Philip Dawid.
\newblock On individual risk.
\newblock \emph{Synthese}, 194\penalty0 (9):\penalty0 3445--3474, 2017.

\bibitem[DeGroot and Fienberg(1981)]{degroot1981assessing}
Morris~H DeGroot and Stephen~E Fienberg.
\newblock Assessing probability assessors: Calibration and refinement.
\newblock Technical report, CARNEGIE-MELLON UNIV PITTSBURGH PA DEPT OF
  STATISTICS, 1981.

\bibitem[DeGroot and Fienberg(1983)]{degroot1983comparison}
Morris~H DeGroot and Stephen~E Fienberg.
\newblock The comparison and evaluation of forecasters.
\newblock \emph{Journal of the Royal Statistical Society: Series D (The
  Statistician)}, 32\penalty0 (1-2):\penalty0 12--22, 1983.

\bibitem[Downey and Hirschfeldt(2010)]{DowneyHirschfeldt}
Rodney~G. Downey and Denis~R. Hirschfeldt.
\newblock \emph{Algorithmic Randomness and Complexity}.
\newblock Theory and Applications of Computability. Springer, 2010.

\bibitem[Dubhashi and Panconesi(2009)]{dubhashi2009concentration}
Devdatt~P Dubhashi and Alessandro Panconesi.
\newblock \emph{Concentration of measure for the analysis of randomized
  algorithms}.
\newblock Cambridge University Press, 2009.

\bibitem[Dwork and Roth(2014)]{dwork2014algorithmic}
Cynthia Dwork and Aaron Roth.
\newblock The algorithmic foundations of differential privacy.
\newblock \emph{Foundations and Trends{\textregistered} in Theoretical Computer
  Science}, 9\penalty0 (3--4):\penalty0 211--407, 2014.

\bibitem[Dwork et~al.(2006)Dwork, McSherry, Nissim, and
  Smith]{dwork2006calibrating}
Cynthia Dwork, Frank McSherry, Kobbi Nissim, and Adam Smith.
\newblock Calibrating noise to sensitivity in private data analysis.
\newblock In \emph{Theory of cryptography conference}, pages 265--284.
  Springer, 2006.

\bibitem[Dwork et~al.(2015)Dwork, Feldman, Hardt, Pitassi, Reingold, and
  Roth]{dwork2015generalization}
Cynthia Dwork, Vitaly Feldman, Moritz Hardt, Toni Pitassi, Omer Reingold, and
  Aaron Roth.
\newblock Generalization in adaptive data analysis and holdout reuse.
\newblock \emph{Advances in Neural Information Processing Systems}, 28, 2015.

\bibitem[Dwork et~al.(2021)Dwork, Kim, Reingold, Rothblum, and
  Yona]{dwork2021outcome}
Cynthia Dwork, Michael~P Kim, Omer Reingold, Guy~N Rothblum, and Gal Yona.
\newblock Outcome indistinguishability.
\newblock In \emph{Proceedings of the 53rd Annual ACM SIGACT Symposium on
  Theory of Computing}, pages 1095--1108, 2021.

\bibitem[Feinberg and Stewart(2008)]{feinberg2008testing}
Yossi Feinberg and Colin Stewart.
\newblock Testing multiple forecasters.
\newblock \emph{Econometrica}, 76\penalty0 (3):\penalty0 561--582, 2008.

\bibitem[Foster and Vohra(1998)]{foster1998asymptotic}
Dean~P Foster and Rakesh~V Vohra.
\newblock Asymptotic calibration.
\newblock \emph{Biometrika}, 85\penalty0 (2):\penalty0 379--390, 1998.

\bibitem[Garg et~al.(2019)Garg, Kim, and Reingold]{garg2019tracking}
Sumegha Garg, Michael~P Kim, and Omer Reingold.
\newblock Tracking and improving information in the service of fairness.
\newblock In \emph{Proceedings of the 2019 ACM Conference on Economics and
  Computation}, pages 809--824, 2019.

\bibitem[Geanakoplos and Polemarchakis(1982)]{geanakoplos1982we}
John~D Geanakoplos and Heraklis~M Polemarchakis.
\newblock We can't disagree forever.
\newblock \emph{Journal of Economic theory}, 28\penalty0 (1):\penalty0
  192--200, 1982.

\bibitem[Globus-Harris et~al.(2022)Globus-Harris, Kearns, and
  Roth]{biasbounties}
Ira Globus-Harris, Michael Kearns, and Aaron Roth.
\newblock An algorithmic framework for bias bounties.
\newblock In \emph{2022 ACM Conference on Fairness, Accountability, and
  Transparency}, FAccT '22, page 1106–1124, New York, NY, USA, 2022.
  Association for Computing Machinery.
\newblock ISBN 9781450393522.
\newblock \doi{10.1145/3531146.3533172}.
\newblock URL \url{https://doi.org/10.1145/3531146.3533172}.

\bibitem[Gupta et~al.(2022)Gupta, Jung, Noarov, Pai, and Roth]{gupta2022online}
Varun Gupta, Christopher Jung, Georgy Noarov, Mallesh~M Pai, and Aaron Roth.
\newblock Online multivalid learning: Means, moments, and prediction intervals.
\newblock In \emph{13th Innovations in Theoretical Computer Science Conference
  (ITCS 2022)}. Schloss Dagstuhl-Leibniz-Zentrum f{\"u}r Informatik, 2022.

\bibitem[H{\'a}jek(2007)]{hajek2007reference}
Alan H{\'a}jek.
\newblock The reference class problem is your problem too.
\newblock \emph{Synthese}, 156\penalty0 (3):\penalty0 563--585, 2007.

\bibitem[H{\'e}bert-Johnson et~al.(2018)H{\'e}bert-Johnson, Kim, Reingold, and
  Rothblum]{hebert2018multicalibration}
Ursula H{\'e}bert-Johnson, Michael Kim, Omer Reingold, and Guy Rothblum.
\newblock Multicalibration: Calibration for the (computationally-identifiable)
  masses.
\newblock In \emph{International Conference on Machine Learning}, pages
  1939--1948. PMLR, 2018.

\bibitem[Jung et~al.(2020)Jung, Ligett, Neel, Roth, Sharifi-Malvajerdi, and
  Shenfeld]{jung2020new}
Christopher Jung, Katrina Ligett, Seth Neel, Aaron Roth, Saeed
  Sharifi-Malvajerdi, and Moshe Shenfeld.
\newblock A new analysis of differential privacy’s generalization guarantees.
\newblock In \emph{11th Innovations in Theoretical Computer Science Conference
  (ITCS 2020)}. Schloss Dagstuhl-Leibniz-Zentrum f{\"u}r Informatik, 2020.

\bibitem[Jung et~al.(2021)Jung, Lee, Pai, Roth, and Vohra]{jung2021moment}
Christopher Jung, Changhwa Lee, Mallesh Pai, Aaron Roth, and Rakesh Vohra.
\newblock Moment multicalibration for uncertainty estimation.
\newblock In \emph{Conference on Learning Theory}, pages 2634--2678. PMLR,
  2021.

\bibitem[Martin{-}L{\"{o}}f(1966)]{Martin-Lof66}
Per Martin{-}L{\"{o}}f.
\newblock The definition of random sequences.
\newblock \emph{Inf. Control.}, 9\penalty0 (6):\penalty0 602--619, 1966.

\bibitem[Marx et~al.(2020)Marx, Calmon, and Ustun]{marx2020predictive}
Charles Marx, Flavio Calmon, and Berk Ustun.
\newblock Predictive multiplicity in classification.
\newblock In \emph{International Conference on Machine Learning}, pages
  6765--6774. PMLR, 2020.

\bibitem[Nau and McCardle(1991)]{nau1991arbitrage}
Robert~F Nau and Kevin~F McCardle.
\newblock Arbitrage, rationality, and equilibrium.
\newblock \emph{Theory and Decision}, 31\penalty0 (2):\penalty0 199--240, 1991.

\bibitem[Sandroni(2003)]{sandroni2003reproducible}
Alvaro Sandroni.
\newblock The reproducible properties of correct forecasts.
\newblock \emph{International Journal of Game Theory}, 32\penalty0
  (1):\penalty0 151--159, 2003.

\bibitem[Sandroni et~al.(2003)Sandroni, Smorodinsky, and
  Vohra]{sandroni2003calibration}
Alvaro Sandroni, Rann Smorodinsky, and Rakesh~V Vohra.
\newblock Calibration with many checking rules.
\newblock \emph{Mathematics of operations Research}, 28\penalty0 (1):\penalty0
  141--153, 2003.

\bibitem[Ville(1939)]{Ville}
Jean Ville.
\newblock \emph{\'{E}tude Critique de la Notion de Collectif}.
\newblock Monographies des Probabilit\'{e}s. Calcul des Probabilit\'{e}s et ses
  Applications. Gauthier-Villars, 1939.

\bibitem[Vovk and Shen(2010)]{VOVK20102632}
Vladimir Vovk and Alexander Shen.
\newblock Prequential randomness and probability.
\newblock \emph{Theoretical Computer Science}, 411\penalty0 (29):\penalty0
  2632--2646, 2010.

\end{thebibliography}
\appendix
\section{Missing Proofs From Section \ref{sec:outofsample}}
\label{app:outofsample}
\outofsample*

\begin{theorem}[Hoeffding's Inequality]
\label{thm:hoeffding}
Lex $X_1,\ldots,X_n$ be independent random variables bounded such that for each $i$, $a_i \leq X_i \leq b_i$. Let $S_n = \sum_{i=1}^n X_i$ denote their sum. Then for all $\eta > 0$:
$$\Pr\left[\left|S_n - \E[S_n] \right| \geq \eta \right] \leq 2\exp\left(\frac{-2\eta^2}{\sum_{i=1}^n (b_i-a_i)^2} \right)$$
\end{theorem}

Hoeffding's inequality is a classical tool in probability and statistics --- a proof can be found e.g. in Chapter 1 of \cite{dubhashi2009concentration}. 

\begin{proof}[Proof of Theorem \ref{thm:reconciler-out-of-sample}]
The bound on $T$ follows directly from Theorem \ref{thm:reconciler-main} without modification. We focus on bounding the Brier scores and the uncertainty region for the resulting models. 

Consider any pair of models $f_1,f_2$. Given a finite dataset $D$ we write $(x,y) \sim D$ to denote uniformly sampling a single datapoint from $D$. We start by comparing $\Pr_{(x,y) \sim D}[x \in U_\epsilon(f_1,f_2)]$ with $\Pr_{(x,y) \sim \cD}[x \in U_\epsilon(f_1,f_2)]$. We have that:
$$\Pr_{(x,y) \sim D}[x \in U_\epsilon(f_1,f_2)] = \frac{1}{n}\sum_{i=1}^n \mathbbm{1}[x_i \in U_\epsilon(f_1,f_2)]$$
Since $\mathbbm{1}[x_i \in U_\epsilon(f_1,f_2)] \in [0,1]$ and $$\E_{D \sim \cD^n}\left[\Pr_{(x,y) \sim D}[x \in U_\epsilon(f_1,f_2)]\right] = \Pr_{(x,y) \sim \cD}[x \in U_\epsilon(f_1,f_2)]$$

we can apply Hoeffding's inequality (Theorem \ref{thm:hoeffding}) to conclude that for every $\eta > 0$:
$$\Pr_{D \sim \cD^n}\left[\left| \Pr_{(x,y) \sim D}[x \in U_\epsilon(f_1,f_2)] -\Pr_{(x,y) \sim \cD}[x \in U_\epsilon(f_1,f_2)]  \right| \geq \eta\right] \leq 2\exp\left(-2\eta^2 n \right)$$
Let $C$ be the set of pairs of models guaranteed in the statement of Lemma \ref{lem:reconciler-model-counting}. Recall that Lemma \ref{lem:reconciler-model-counting} guarantees us that $|C| \leq (4(m+1)^{32/\alpha\epsilon^2 + 1}$. We can apply the union bound to all pairs of models $(f_1,f_2) \in C$ to conclude that with probability at least $1 - 2|C|\exp\left(-2\eta^2 n \right)$ (over the randomness of $D$) we have that for \emph{every} pair $(f_1,f_2) \in C$:
$$\left| \Pr_{(x,y) \sim D}[x \in U_\epsilon(f_1,f_2)] -\Pr_{(x,y) \sim \cD}[x \in U_\epsilon(f_1,f_2)]  \right| \leq \eta$$
Choosing $$\eta = \sqrt{\frac{\log \left(\frac{2|C|}{\delta} \right)}{2n}}$$
we get that with probability $1-\delta$ over the draw of $D$, for every pair $(f_1,f_2) \in C$:
\begin{eqnarray*}\left| \Pr_{(x,y) \sim D}[x \in U_\epsilon(f_1,f_2)] -\Pr_{(x,y) \sim \cD}[x \in U_\epsilon(f_1,f_2)]  \right| &\leq&  \sqrt{\frac{\log \left(\frac{2|C|}{\delta} \right)}{2n}} \\
&\leq& \sqrt{\frac{\left(\frac{32}{\alpha\epsilon^2}+1\right)\log\left(\frac{8\left(\lceil \frac{2}{\sqrt{\alpha}\epsilon}\rceil+1\right)}{\delta}\right)}{n}}
\end{eqnarray*}
where the first inequality follows from plugging in our definition of $\eta$ and the final inequality follows from plugging in our bound on $|C|$ and the definition of $m$. 

Because we know from Theorem \ref{thm:reconciler-main} that the models $f_1^{T_1},f_2^{T_2}$ output by Algorithm \ref{alg:reconciler} satisfy that $\Pr_{(x,y) \sim D}[x \in U_\epsilon(f_1^{T_1},f_2^{T_2})] \leq \alpha$ we can conclude that with probability $1-\delta$:
$$\Pr_{(x,y) \sim \cD}[x \in U_\epsilon(f_1^{T_1},f_2^{T_2})] \leq \sqrt{\frac{\left(\frac{32}{\alpha\epsilon^2}+1\right)\log\left(\frac{8\left(\lceil \frac{2}{\sqrt{\alpha}\epsilon}\rceil+1\right)}{\delta}\right)}{n}}$$

We can bound the Brier score of the resulting models in exactly the same way. For any fixed model $f:\cX\rightarrow [0,1]$, we can write the empirical Brier score (i.e. the Brier score as evaluated over $D$) as:
$$B(f,D) = \frac{1}{n}\sum_{i=1}^n (f(x_i)-y_i)^2$$ 
Since $(f(x_i) - y_i)^2 \in [0,1]$ and $\E_{D \sim \cD^n}[B(f,D)]=B(f,\cD)$, we can apply Hoeffding's inequality (Theorem \ref{thm:hoeffding}) exactly as before to conclude that for every pair of models $(f_1^{T_1},f_2^{T_2}) \in C$, with probability $1-\delta$:
$$\left|B_D(f_1^{T_1}) - B(f_1^{T^1}) \right| \leq  \sqrt{\frac{\left(\frac{16}{\alpha\epsilon^2}+1\right)\log\left(\frac{16\left(\lceil \frac{2}{\sqrt{\alpha}\epsilon}\rceil+1\right)}{\delta}\right)}{n}}$$
and with probability $1-\delta$:
$$\left|B_D(f_2^{T_2}) - B(f_2^{T_2}) \right| \leq  \sqrt{\frac{\left(\frac{16}{\alpha\epsilon^2}+1\right)\log\left(\frac{16\left(\lceil \frac{2}{\sqrt{\alpha}\epsilon}\rceil+1\right)}{\delta}\right)}{n}}$$
Observe that the same holds true for the original pair of models $(f_1,f_2)$, since $(f_1,f_2) \in C$ (they correspond to the models output after transcripts of length 0). We further know from Theorem \ref{thm:reconciler-main} that:
$B(f_1^{T_1},D) \leq B(f_1,D) - T_1 \cdot  \frac{\alpha \epsilon^2}{16}$ and $B(f_2^{T_2}<D) \leq B(f_2,D) - T_2 \cdot  \frac{\alpha \epsilon^2}{16}$.

Instantiating these bounds for the four models $\{f_1,f_2,f_1^{T_1},f_2^{T_2}\}$, and setting $\delta \leftarrow \delta/4$ so that we can union bound over all four models, we have that with probability $1-\delta$ that we simultaneously have:
$$B(f_1^{T_1},\cD) \leq B(f_1,\cD) - T_1 \cdot  \frac{\alpha \epsilon^2}{16} + 2\sqrt{\frac{\left(\frac{16}{\alpha\epsilon^2}+1\right)\log\left(\frac{64\left(\lceil \frac{2}{\sqrt{\alpha}\epsilon}\rceil+1\right)}{\delta}\right)}{n}}$$
$$B(f_2^{T_1},\cD) \leq B(f_2,\cD) - T_1 \cdot  \frac{\alpha \epsilon^2}{16} + 2\sqrt{\frac{\left(\frac{16}{\alpha\epsilon^2}+1\right)\log\left(\frac{64\left(\lceil \frac{2}{\sqrt{\alpha}\epsilon}\rceil+1\right)}{\delta}\right)}{n}}$$
\end{proof}

\section{Proofs from Section \ref{sec:contestable}}
\label{app:contestable}

\begin{definition}[\cite{dwork2006calibrating}]
Let $A:\cZ^n \rightarrow R$ be any randomized algorithm. Say that two datasets $D, D' \in \cZ^n$ are \emph{neighboring} if they differ in exactly one element. $A$ is $(\epsilon,\delta)$-differentially private if for every pair of neighboring $D, D' \in \cZ^n$ and every measurable subset of the range of $A$, $S \subseteq R$, $A$ satisfies:
$$\Pr[A(D) \in S]\leq \exp(\epsilon)\Pr[A(D') \in S] + \delta$$
where the probability is taken over the randomness of $A$.
\end{definition}

\begin{lemma}[\cite{dwork2014algorithmic}]
\label{lem:sampleaccurate}
Algorithm \ref{alg:contest} ($\mathbf{f}$.contest) satisfies $(\epsilon,\delta)$ differential privacy with respect to its dataset $D$, and has the following accuracy properties for any sequence of $K$ contestations $g_1,\ldots,g_K$. For each contestation $g_t$ before the algorithm has halted (i.e. while $c < C$), then with probability $1-\delta$ simultaneously for every contestation $g_t$:
\begin{enumerate}
    \item If $\hat \eta_t \geq \hat T_c$ (the contestation is accepted) then:
    $$\left|\tilde \mu_t - \frac{1}{n}\sum_{(x,y) \in D}g_t(x)\right|\leq \frac{2(\ln K + \ln(4C/\delta))\sqrt{C\ln(2/\delta)}(\sqrt{512}+1)}{\epsilon n}$$
    $$\left|\eta_t(f_c,g_t) - \tilde \eta_t\right|\leq \frac{2(\ln K + \ln(4C/\delta))\sqrt{C\ln(2/\delta)}(\sqrt{512}+1)}{\epsilon n}$$
    $$|\eta_t(f_c,g_t)| \geq T - \frac{2(\ln K + \ln(4C/\delta))\sqrt{C\ln(2/\delta)}(\sqrt{512}+1)}{\epsilon n}$$
    \item If $\hat \eta_t < T_c$ (The contestation is rejected) then: 
    $$|\eta_t(f_c,g_t)| \leq T + \frac{2(\ln K + \ln(4C/\delta))\sqrt{C\ln(2/\delta)}(\sqrt{512}+1)}{\epsilon n}$$
\end{enumerate}
\end{lemma}
\begin{proof}
Algorithm \ref{alg:contest} is an instantiation of Algorithm 3 (NumericSparse) from \cite{dwork2014algorithmic}. Lemma 5.1 follows from the analysis of NumericSparse (Theorem 3.2.8 in \cite{dwork2014algorithmic}). 
\end{proof}

Lemma \ref{lem:sampleaccurate} tells us that Algorithm \ref{alg:contest} is both differentially private and \emph{sample-accurate} --- it produces estimates $\tilde \mu_t$ and $\tilde \eta _t$ that are close to the quantities they are supposed to be estimating, or else reveals that $\eta_t(f_c,g_t)$ is small. Next we quote a version of the ``transfer theorem'' that tells us that these two properties together give us guarantees of out of sample accuracy.

\begin{theorem}[Application of Theorem 3.5 in \cite{jung2020new}]
\label{thm:transfer}
Let $D \sim \cD^n$ be a dataset consisting of $n$ i.i.d. samples from some underlying distribution $\cD$ over $\cZ$. Let $c > 0$ be any number. Suppose an algorithm $A$ is $(\epsilon,\delta)$ differentially private in $D$ and for each round $t$ either produces estimates $\tilde \mu_t$, $\tilde \eta_t$ that with probability $1-\delta$ are within $\alpha$ of their sample estimates on $D$:
    $$\left|\tilde \mu_t - \frac{1}{n}\sum_{(x,y) \in D}g_t(x)\right|\leq \alpha$$
    $$\left|\eta_t(f_c,g_t) - \tilde \eta_t\right|\leq \alpha$$
    or declines to answer, in which case it promises that:
    $$|\eta_t(f_c,g_t)| \leq T + \alpha$$
    Then, with probability $1-\frac{2\delta}{c}$, the estimates $\tilde \mu_t$, $\tilde \eta_t$ are also accurate as evaluated on $\cD$:
        $$\left|\tilde \mu_t - \E_{(x,y) \sim \cD}[g_t(x)]\right|\leq \alpha + (e^\epsilon-1) + 3c$$
    $$\left|\E_{(x,y) \sim \cD}[g_t(x)(y-f_c(x))] - \tilde \eta_t\right|\leq \alpha + (e^\epsilon-1) + 3c$$
    and for any round on which the algorithm declines to provide an answer:
    $$\left|\E_{(x,y) \sim \cD}[g_t(x)(y-f_c(x))]\right| \leq T + \alpha  + (e^\epsilon-1) + 3c$$
\end{theorem}

Observing that for any $\epsilon < 1$ we have $e^\epsilon - 1 \in  O(\epsilon)$, we can combine Lemma \ref{lem:sampleaccurate} with Theorem \ref{thm:transfer} (choosing $c = 1/n$) to obtain the following corollary:
\begin{corollary} 
\label{cor:guessncheck}
Let $D \sim \cD^n$ be a dataset consisting of $n$ i.i.d. samples from some underlying distribution $\cD$ over $\cZ$. Let $\epsilon,\delta  \leq 1$. Algorithm \ref{alg:contest} ($\mathbf{f}$.contest) run with privacy parameters $\epsilon, \delta$ has the following accuracy guarantees. For any sequence of $K$ contestations $g_1,\ldots,g_K$, with probability $1-2\delta n$ simultaneously for every contestation:
\begin{enumerate}
  \item If $\hat \eta_t \geq \hat T_c$ (the contestation is accepted) then:
    $$\left|\tilde \mu_t - \E_{(x,y) \sim \cD}[g_t(x)]\right| \in O\left(\frac{\log \frac{K}{\delta}\sqrt{C\ln\frac{1}{\delta}}}{\epsilon n} + \epsilon \right)$$
    $$\left|\E_{(x,y) \sim \cD}[g_t(x)(y-f_c(x))]  - \tilde \eta_t\right|\in O\left(\frac{\log \frac{K}{\delta}\sqrt{C\ln\frac{1}{\delta}}}{\epsilon n} + \epsilon \right)$$
    $$\left|\E_{(x,y) \sim \cD}[g_t(x)(y-f_c(x))]\right| \geq T -  O\left(\frac{\log \frac{K}{\delta}\sqrt{C\ln\frac{1}{\delta}}}{\epsilon n} + \epsilon \right)$$
    \item If $\hat \eta_t < T_c$ (The contestation is rejected) then: 
    $$\left|\E_{(x,y) \sim \cD}[g_t(x)(y-f_c(x))]\right|\leq T +  O\left(\frac{\log \frac{K}{\delta}\sqrt{C\ln\frac{1}{\delta}}}{\epsilon n} + \epsilon \right)$$
\end{enumerate}

Choosing $\epsilon =\sqrt{\frac{\log \frac{K}{\delta}\sqrt{C\ln\frac{1}{\delta}}}{ n}}$ makes all of the error terms evaluate to 
$$O\left(\sqrt{\frac{\log \frac{K}{\delta}\sqrt{C\ln\frac{1}{\delta}}}{n}} \right)$$
\end{corollary}

\contestable*
\begin{proof}
It is immediate from Corollary \ref{cor:guessncheck} and the definition of $T$ that for accepted contestations, $\E_{(x,y) \sim \cD}[g_t(x)(y-f_c(x))] \geq \Omega\left( \frac{\left(\log \frac{K}{\delta}\right)^{1/3}\left( \ln \frac{1}{\delta}\right)^{1/6}}{n^{1/3}}\right)$, and for rejected contestations $g_t$, $\E_{(x,y) \sim \cD}[g_t(x)(y-f_c(x))] \leq O\left( \frac{\left(\log \frac{K}{\delta}\right)^{1/3}\left( \ln \frac{1}{\delta}\right)^{1/6}}{n^{1/3}}\right)$. It remains to establish the other conclusions of the theorem. 

Fix any round $t$ for which a contestation is accepted $(\hat \eta_t \geq \hat T_c$). We have that:
\begin{eqnarray*}
\frac{\tilde \eta_t}{\tilde \mu_t} &\leq& \frac{\E_{(x,y) \sim \cD}[g_t(x)(y-f_c(x))] + O\left(\sqrt{\frac{\log \frac{K}{\delta}\sqrt{C\ln\frac{1}{\delta}}}{n}} \right)}{\E_{(x,y) \sim \cD}[g_t(x)] - O\left(\sqrt{\frac{\log \frac{K}{\delta}\sqrt{C\ln\frac{1}{\delta}}}{n}} \right)}
\end{eqnarray*}

Recall that we know that:
$$\left|\E_{(x,y) \sim \cD}[g_t(x)(y-f_c(x))]\right|\geq T -  O\left(\sqrt{\frac{\log \frac{K}{\delta}\sqrt{C\ln\frac{1}{\delta}}}{n}} \right)$$
and so in particular we also know that:
$$\E_{(x,y) \sim \cD}[g_t(x)]\geq T -  O\left(\sqrt{\frac{\log \frac{K}{\delta}\sqrt{C\ln\frac{1}{\delta}}}{n}} \right)$$

Since we can take $T = \Omega\left(\sqrt{\frac{\log \frac{K}{\delta}\sqrt{C\ln\frac{1}{\delta}}}{n}} \right)$, this means that we have: 
$$\frac{\tilde \eta_t}{\tilde \mu_t} \leq  \frac{3}{2}\frac{\E_{(x,y) \sim \cD}[g_t(x)(y-f_c(x))]} {\E_{(x,y) \sim \cD}[g_t(x)]} = \frac{3}{2}\E[(y-f_c(x) | g_t(x) = 1]$$
and similarly:
$$\frac{\tilde \eta_t}{\tilde \mu_t} \geq  \frac{1}{2}\frac{\E_{(x,y) \sim \cD}[g_t(x)(y-f_c(x))]} {\E_{(x,y) \sim \cD}[g_t(x)]} = \frac{1}{2}\E[(y-f_c(x) | g_t(x) = 1]$$

Now consider any contestation $g_t$ that is accepted ($\hat \eta_t \geq \hat T_c$). With probability $1-2\delta n$, we have that for all such $g_t$, 
$$\left|\E_{(x,y) \sim \cD}[g_t(x)(y-f_c(x))]\right| \geq T - O\left(\sqrt{\frac{\log \frac{K}{\delta}\sqrt{C\ln\frac{1}{\delta}}}{n}} \right) \geq \Omega(T)$$
Recall that $\tilde \Delta_t = \frac{\tilde \eta_t}{\tilde \mu_t}$, and let $\Delta_t = \E[(y-f_c(x) | g_t(x) = 1]$ Let $\hat f_{c+1} = h(x,f_c,g_t,\Delta_t)$ be the update that \emph{would} have resulted had the algorithm used $\Delta_t$ rather than $\tilde \Delta_t$ as part of its update. 
Two applications of Lemma \ref{lem:reconciler-progress} then give:
\begin{eqnarray*}
B(f_c,\cD) - B(f_{c+1},\cD) &=& (B(f_c,\cD) - B(\hat f_{c+1})) - (B(f_{c+1} - B(\hat f_{c+1})) \\
&=& \mu(g_t)\cdot \Delta_t^2 - \mu(g_t) \cdot (\Delta_t - \tilde \Delta_t)^2 \\
&\geq& \mu(g_t)\cdot \left(\Delta_t^2 - \left(\frac{\Delta_t}{2}\right)^2\right) \\
&=& \frac{3}{4}\mu(g_t)\Delta_t^2 \\
&\geq& \frac{3}{4}\mu(g_t)^2 \Delta_t^2 \\
&\geq& \Omega(T^2)
\end{eqnarray*}
Since for all $C$, $0 \leq B(f_c,D) \leq 1$,
 with probability $1-2\delta/n$, there can be at most $C = O(1/T^2)$ updates. Since we must take: 
 $T = \Omega\left(\sqrt{\frac{\log \frac{K}{\delta}\sqrt{C\ln\frac{1}{\delta}}}{n}} \right) = \Omega\left( \frac{\left(\log \frac{K}{\delta}\right)^{1/2}\left(C \ln \frac{1}{\delta}\right)^{1/4}}{\sqrt{n}}\right)$, we can set $C = \theta(1/T^2)$ so long as $T^{3/2} = \Theta\left( \frac{\left(\log \frac{K}{\delta}\right)^{1/2}\left( \ln \frac{1}{\delta}\right)^{1/4}}{\sqrt{n}}\right)$, or:
 $$T = \Theta\left( \frac{\left(\log \frac{K}{\delta}\right)^{1/3}\left( \ln \frac{1}{\delta}\right)^{1/6}}{n^{1/3}}\right)$$
\end{proof}

\end{document}